\documentclass[twoside]{article}

\usepackage[accepted]{aistats2026}
\usepackage{array}

\usepackage[round]{natbib}

\usepackage{fontawesome5}
\usepackage{hyperref}

\usepackage[linesnumbered,norelsize,ruled]{algorithm2e}
\usepackage{amsmath}
\usepackage{amssymb}
\usepackage{amsthm}
\usepackage{bm}
\usepackage{booktabs}
\usepackage[font=footnotesize,labelfont=bf]{caption}
\usepackage[capitalize,nameinlink]{cleveref}
\usepackage{color}
\usepackage{dsfont}
\usepackage[T1]{fontenc}
\usepackage{listings}
\usepackage{makecell}
\usepackage{mathtools}
\usepackage{mathrsfs}
\usepackage[xcolor,framemethod=tikz]{mdframed}
\usepackage{multirow}
\usepackage{pifont}
\usepackage{rotating}
\usepackage[group-separator={,},group-minimum-digits={3}]{siunitx}
\usepackage{subcaption}
\usepackage{tabularx}
\usepackage{tikz}
\usepackage{wrapfig}
\usepackage{xcolor}

\usetikzlibrary{arrows.meta}
\usetikzlibrary{calc}
\usetikzlibrary{positioning}

\mdfdefinestyle{myframe}{
  linecolor=blue!30,
  linewidth=1pt,
  backgroundcolor=blue!5,
  innertopmargin=.5\baselineskip,
  innerleftmargin=0.8em,
  innerrightmargin=0.8em,
  innerbottommargin=.5\baselineskip,
  frametitlebackgroundcolor=blue!30,
}

\usepackage{colortbl}
\usepackage{diagbox}
\newcolumntype{C}{>{\columncolor[gray]{0.85}}c}

\hypersetup{
  colorlinks=true,
  linkcolor={red!15!green!35!blue!95},
  citecolor={green!50!black},
  urlcolor={red!50!black}
}

\DontPrintSemicolon
\SetArgSty{textnormal}
\SetKwInOut{Input}{Input}
\SetKwInOut{Output}{Output}
\SetKwComment{Comment}{$\triangleright$\ }{}
\IncMargin{1em}  %

\renewcommand{\tilde}{\widetilde}
\renewcommand{\hat}{\widehat}

\newcommand{\Bcal}{\mathcal{B}}

\newcommand{\Pcal}{\mathcal{P}}

\newcommand{\Ucal}{\mathcal{U}}

\newcommand{\Xcal}{\mathcal{X}}

\newcommand{\Ebb}{\mathbb{E}}

\newcommand{\Lbb}{\mathbb{L}}

\newcommand{\Nbb}{\mathbb{N}}

\newcommand{\Rbb}{\mathbb{R}}

\newcommand{\Cbf}{\mathbf{C}}

\newcommand{\Ibf}{\mathbf{I}}

\newcommand{\Pbf}{\mathbf{P}}

\newcommand{\Sbf}{\mathbf{S}}

\newcommand{\Ubf}{\mathbf{U}}

\newcommand{\Wbf}{\mathbf{W}}

\newcommand{\Ybf}{\mathbf{Y}}

\newcommand{\bbf}{\mathbf{b}}

\newcommand{\fbf}{\mathbf{f}}
\newcommand{\gbf}{\mathbf{g}}

\newcommand{\pbf}{\mathbf{p}}
\newcommand{\qbf}{\mathbf{q}}

\newcommand{\ubf}{\mathbf{u}}
\newcommand{\vbf}{\mathbf{v}}
\newcommand{\wbf}{\mathbf{w}}
\newcommand{\xbf}{\mathbf{x}}
\newcommand{\ybf}{\mathbf{y}}
\newcommand{\zbf}{\mathbf{z}}

\newcommand{\oneds}{\mathds{1}}

\newcommand{\defeq}{\coloneqq}

\renewcommand{\subset}{\subseteq}
\renewcommand{\supset}{\supseteq}

\DeclareMathOperator*{\E}{\Ebb}

\DeclareMathOperator*{\argmin}{arg\,min}

\DeclareMathOperator{\supp}{\mathrm{supp}}

\DeclareMathOperator{\convhull}{\mathrm{conv}}

\newcommand{\set}[1]{\left\lbrace{#1}\right\rbrace}
\newcommand{\setcomp}[2]{\left\lbrace{#1} \hspace{0.02em}\relmiddle{|}\hspace{0.02em} {#2}\right\rbrace}
\newcommand{\inpr}[2]{\left\langle{#1},{#2}\right\rangle}
\newcommand{\relmiddle}[1]{\mathrel{}\middle#1\mathrel{}}

\newcommand{\rd}{\,\mathrm{d}}

\renewcommand{\epsilon}{\varepsilon}
\newcommand{\frob}{{\mathrm{F}}}

\theoremstyle{plain}
\newtheorem{theorem}{Theorem}
\newtheorem{definition}{Definition}

\newtheorem{proposition}{Proposition}

\theoremstyle{definition}

\definecolor{codegreen}{rgb}{0,0.6,0}
\definecolor{codegray}{rgb}{0.5,0.5,0.5}
\definecolor{codepurple}{rgb}{0.58,0,0.82}
\definecolor{backcolour}{rgb}{0.95,0.95,0.92}

\lstdefinestyle{mylst}{
    backgroundcolor=\color{backcolour},   
    commentstyle=\color{codegreen},
    keywordstyle=\color{magenta},
    numberstyle=\tiny\color{codegray},
    stringstyle=\color{codepurple},
    basicstyle=\ttfamily\footnotesize,
    breakatwhitespace=false,         
    breaklines=true,                 
    captionpos=b,                    
    keepspaces=true,                 
    numbers=left,
    numbersep=5pt,
    xleftmargin=15pt,
    showspaces=false,                
    showstringspaces=false,
    showtabs=false,                  
    tabsize=2,
    frame=tb,
}

\newcommand{\method}[1]{\textsf{#1}}

\begin{document}

\runningtitle{Brenier Isotonic Regression}
\runningauthor{Bao, Eshraghi, Wang}

\twocolumn[

\aistatstitle{Brenier Isotonic Regression}

\aistatsauthor{ Han Bao${}^1$ \And Amirreza Eshraghi${}^2$ \And Yutong Wang${}^2$ }

\aistatsaddress{ ${}^1$The Institute of Statistical Mathematics \And ${}^2$Illinois Institute of Technology }

]

\begin{abstract}
Isotonic regression (IR) is a shape-constrained regression to maintain a univariate fitting curve non-decreasing, which has numerous applications.
When it comes to multivariate responses, IR is no longer applicable because monotonicity is not readily extendable.
We consider a multi-output regression problem where a regression function is \emph{cyclically monotone}.
Roughly speaking, a cyclically monotone function is the gradient of some convex potential.
Whereas enforcing cyclic monotonicity is apparently challenging, we leverage the fact that Kantorovich's optimal transport (OT) always yields a cyclically monotone coupling.
This naturally allows us to interpret a regression function and the convex potential as a link function in generalized linear models and Brenier's potential in OT, respectively.
We call this IR extension \emph{Brenier isotonic regression}.
We demonstrate applications to probability calibration and single-index models.
\end{abstract}

\section{INTRODUCTION}
\label{section:introduction}
Given a labeled dataset $\set{(z_i,y_i)}_{i\in[n]}$ with an input $z_i\in\Rbb$ and response $y_i\in\Rbb$,
isotonic regression (IR) returns nonparametric estimates $\set{\hat{y}_i}\subset\Rbb$ to solve
\[
\begin{aligned}
  &\min_{\hat y_1,\dots,\hat y_n\in\Rbb} \sum_{i\in[n]}(y_i-\hat y_i)^2 \\
  &\text{subject to} \quad (z_i - z_j)(\hat y_i - \hat y_j) \ge 0 \quad \forall (i,j)\in[n]^2.
\end{aligned}
\]
IR is nonparametric, imposing no regression model but only enforcing monotonicity.
This specific quadratic program can be solved efficiently by the Pool Adjacent Violaters (PAV) algorithm~\citep{Ayer1955},
which achieves the optimal solution of IR provably~\citep{Robertson1988,Best1990}.
IR has been extensively used in probability calibration of binary classifiers~\citep{Zadrozny2002KDD,Niculescu2005ICML} and single-index models~\citep{Kalai2009COLT,Kakade2011NeurIPS}.
Moreover, IR has been recently used for the Bayes error estimation~\citep{Ushio2025}.

When both inputs and responses become multivariate, the extension of IR becomes scarce.
Consider a labeled dataset $\set{(\zbf_i,\ybf_i)}_{i\in[n]}$ with $\zbf_i,\ybf_i\in\Rbb^d$,
then what is a natural extension of univariate monotonicity?

Such a multi-output ``IR'' is essential in probability calibration.
In binary classification, a covariate $\xbf_i$ is classified based on a univariate probability $\hat p_i\in[0,1]$ representing the confidence of a binary label $y_i\in\set{0,1}$.
IR applied to $\set{(\hat p_i,y_i)}_{i\in[n]}$ typically ameliorates the quality of the confidence.
Yet, multiclass classification deals with multinomial probability $\hat{\pbf}_i\in\triangle^{d-1}$ representing the confidence of a one-hot label $\ybf_i\in\set{0,1}^d$,
where $\triangle^{d-1}$ is the probability simplex.
Since IR is limited to 1D responses, the one-vs-rest (OvR) formulation has been applied in this case~\citep{Zadrozny2002KDD,Niculescu2005ICML}.
The OvR formulation requires an additional normalization step after learning a regressor for each class, and what is worse, each regressor is independent without taking care of class correlation~\citep{Arad2025ICML}.

To extend IR to the multi-class case, let us first consider generalized linear models (GLMs) for multinomial outputs~\citep{Agarwal2014ICML,Lam2023ICML}:
\begin{equation}
  \label{equation:glm}
  \E[Y|X=\xbf] = \nabla\Phi(\Wbf^*\xbf) := \varphi(\Wbf^*\xbf),
\end{equation}
where $\Wbf^*\in\Rbb^{d\times D}$ is the weight matrix, $\Phi: \Rbb^d \to \Rbb$ is a proper convex lower-semicontinuous (l.s.c.) function, $\varphi:\Rbb^d\to\Rbb^d$ is the inverse link function, and $Y\in\set{0,1}^d$ is a one-hot class vector.
The choice $\varphi=\nabla\Phi$ is standard in literature of canonical proper losses~\citep{Buja2005,Gneiting2007JASA,Williamson2016JMLR,Bao2024}.
In view of Fenchel--Young losses~\citep{Blondel2020JMLR,Bao2021AISTATS}, the convex potential $\Phi$ can be regarded as a generalized log-partition function or free energy.
The virtue of multinomial GLMs~\eqref{equation:glm} is rooted in the following fact:
$\Phi$ is proper l.s.c. convex if and only if $\varphi$ has a \emph{cyclically monotone} graph~\citep{Rockafellar1966}.
Therefore, we consider regression with an unknown but cyclically monotone link, which we call \emph{cyclically monotonic isotonic regression (CMIR)}:
\begin{equation}
  \label{equation:cm_ir}
  \begin{aligned}
    & \min_{\hat{\ybf}_1,\dots,\hat{\ybf}_n \in \triangle^{d-1}} \sum_{i\in[n]} \|\ybf_i - \hat{\ybf}_i\|_2^2 \quad \text{subject to} \\
    & \text{$\hat{\ybf}_i = \varphi(\zbf_i)$ for some cyclically monotone $\varphi$}.
  \end{aligned}
\end{equation}

We propose a \emph{nonparametric} solution to CMIR~\eqref{equation:cm_ir}.
Thus far, \citet{Agarwal2014ICML} models the inverse link $\varphi$ with a finite number of bases, and \citet{Lam2023ICML} models the convex potential $\Phi$ with input-convex neural networks~\citep{Amos2017ICML},
but both are \emph{parametric} and may be subject to the approximation error.
To nonparametrically solve CMIR~\eqref{equation:cm_ir}, we leverage an elegant connection between convexity and \emph{optimal transport (OT)}.
Specifically, the celebrated Brenier's theorem states that an optimal transport map can be represented by the gradient of a convex potential~\citep{Brenier1991}.
Thus, we can consider first transporting $\set{\zbf_i}$ to $\set{\hat{\ybf}_i}$ and then minimizing the regression error---%
reformulating CMIR into a bi-level optimization with OT being the inner problem.
Inspired by Brenier's theorem, we call our approach \emph{Brenier isotonic regression} (\method{BrenierIR}).
We implement this bi-level program with \texttt{scipy} in an end-to-end fashion, which estimates the implicit gradient with the finite difference method~\citep{2020SciPy-NMeth}.

Experimentally, we test \method{BrenierIR} on calibration and single-index model tasks.
In particular, it performs surprisingly well and stably on the calibration task.
\method{BrenierIR} is more principled with the connection to GLMs, and does not require complicated hyperparameter tuning.
Hence, we suggest practitioners to test \method{BrenierIR} when calibrating multiclass classifiers.

\subsection{Additional related work}
\citet{Han2019AOS} and \citet{Fang2021AOS} considered IR with multivariate covariates, but the response remains real-valued and coordinate-wise monotonicity is considered.
\citet{Sasabuchi1983} considered multi-output IR, but again coordinate-wise monotonicity is considered.
The main drawback of coordinate-wise monotonicity
is that it does not capture GLMs. Even in the simplest case of multinomial logistic regression, the softmax is not coordinate-wise monotone but \emph{is}
cyclically monotone~\citep{Gao2017}.

The connection between monotonicity and OT has been the basis of conditional vector quantile~\citep{Carlier2016AOS,Chernozhukov2017}.
In particular, vector quantile estimation has been actively studied in terms of scalability~\citep{Rosenberg2023ICLR}, statistical property~\citep{Hallin2021AOS}, and neural network modeling~\citep{Kondratyev2025}.
While they focus on \emph{estimating} conditional vector quantiles, we optimize the regression error by \emph{fitting} vector quantiles.
To our knowledge, \citet{Dawkins2022AAAI} is the only previous work focusing on this monotone fitting problem with an application in economics, while they mainly focus on the PAC learnability analysis.
The connection to OT is absent in their work, and establishing this connection is a central contribution of ours.

The connection between OT and regression might be apparently elusive because OT deals with two ``uncoupled'' sets but regression cares about ``coupled'' in-out relations.
In fact, OT has been applied only to uncoupled isotonic regression~\citep{Rigollet2019,Slawski2024JMLR}---%
the unused result \citet[Proposition~2.3]{Rigollet2019} interestingly demonstrates that the pushforward is an isometry between the $L_p$- and $p$-Wasserstein distances.
\citet{Paty2020AISTATS} implied a connection between (coupled) isotonic regression and OT but only in 1D case.
\citet{Berta2024AISTATS} extended IR to the multiclass case by preserving AUC.
However, the multiclass extension of AUC lacks a canonical definition, which limits the approach.
Thus, our work fills the missing piece in these lines of work.

\section{BACKGROUND}
\label{section:background}

\paragraph*{Notation.}
$\mathscr{P}(\Rbb^d)$ denotes the set of Borel probability measures.
The support of $\mu\in\mathscr{P}(\Rbb^d)$ is denoted by $\supp(\mu)\defeq\setcomp{\zbf\in\Rbb^d}{\mu(\zbf)>0}$.
We write the Dirac measure by $\delta_\bullet$.
We write the all-ones vector by $\oneds_n\in\Rbb^n$.
The probability simplex is denoted by $\triangle^{d-1}\defeq\setcomp{\pbf\in\Rbb^d_{\ge0}}{\inpr{\pbf}{\oneds_d}=1}$.
The set of permutations over $[n]$ is written by $\mathrm{Perm}(n)$.
The Lebesgue space $L^p(\mu)$ is the space of measurable functions whose $p$-th moment with respect to a measure $\mu$ is finite.

\paragraph*{Optimal transport.}
Given a cost $c:\Rbb^d\times\Rbb^d\to\Rbb$ and probability measures $\mu,\nu\in\mathscr{P}(\Rbb^d)$,
the Monge problem is an OT problem formulated as follows:
\begin{equation}
  \label{equation:monge}
  \inf_{T} \setcomp{ \int_{\Rbb^d} c(\zbf,T(\zbf))\, \rd\mu(\zbf) }
  { T_\sharp\mu = \nu },
\end{equation}
where $T:\Rbb^d\to\Rbb^d$ is a measurable map and $T_\sharp\mu$ denotes the pushforward of $\mu$ by $T$,
defined as $(T_\sharp\mu)(A) = \mu(T^{-1}(A))$ for every Borel $A\subseteq\Rbb^d$.
If $T^*$ attains the infimum, $T^*$ is called the optimal transport map.

We typically relax the challenging constraint in the Monge problem by the following Kantorovich problem:
\begin{equation}
  \label{equation:kantorovich}
  \begin{aligned}
    & \inf_\pi \setcomp{ \int_{\Rbb^d \times \Rbb^d} 
      c(\zbf,\ubf) \, \rd\pi(\zbf,\ubf)  }
      {\pi \in \Ucal(\mu,\nu) }, \\
    & \text{where} \quad 
      \Ucal(\mu,\nu) \defeq 
      \setcomp{ \pi }{ \Pi_{1\sharp}\pi = \mu,\; \Pi_{2\sharp}\pi = \nu }.
  \end{aligned}
\end{equation}

and $\Pi_{1\sharp}$ and $\Pi_{2\sharp}$ are the pushforward by the projections $\Pi_1(\zbf,\ubf)=\zbf$ and $\Pi_2(\zbf,\ubf)=\ubf$, respectively.
This is an infinite-dimensional linear program over $\mathscr{P}(\Rbb^d\times\Rbb^d)$,
and the solution $\pi_*$ is called the optimal coupling of $\mu$ and $\nu$ (whenever it exists).

The relationship between the Monge and Kantorovich problems is clear for the discrete case.
Consider the discrete measures $\mu=\frac1n\sum_{i\in[n]}\delta_{\zbf_i}$ and $\nu=\frac1n\sum_{j\in[n]}\delta_{\ubf_j}$.
Let $\Cbf\in\Rbb^{n\times n}$ be the cost defined by $C_{ij}\defeq c(\zbf_i,\ubf_j)$.
The discrete Monge problem is
\begin{equation}
  \label{equation:discrete_monge}
  \min_{\sigma\in\mathrm{Perm}(n)}\frac1n\sum_{i\in[n]}C_{i\sigma(i)}.
\end{equation}
Further, let us introduce the (scaled) Birkhoff polytope
\[
  \Bcal(m,n)\!\defeq\!\setcomp{\Pbf\in\Rbb_{\ge0}^{m\times n}\!}{\!m\Pbf\oneds_n = \oneds_m, n\Pbf^\top\oneds_m = \oneds_n}\!.
\]
The discrete Kantorovich problem is
\begin{equation}
  \label{equation:discrete_kantorovich}
  \min_{\Pbf\in\Bcal(n,n)}\inpr{\Cbf}{\Pbf},
\end{equation}
where $\inpr{\Cbf}{\Pbf}\defeq\sum_{i,j}C_{ij}P_{ij}$ is the transport cost.
The polytope $\Bcal(n,n)$ enforces $n\Pbf$ to be doubly stochastic.
The following well-known/celebrated proposition states that an optimizer for the Kantorovich problem~\eqref{equation:discrete_kantorovich} is a permutation matrix.
For a permutation $\sigma\in\mathrm{Perm}(n)$, we write $\Pbf^\sigma$ for the permutation matrix:
\[
  \forall (i,j)\in[n]^2, \quad P^\sigma_{ij} = \begin{cases}
    1/n & \text{if $j=\sigma_i$,} \\
    0 & \text{otherwise.}
  \end{cases}
\]
\begin{proposition}[{\citet[Proposition~1.3.1]{Panaretos2020}}]
  \label{proposition:permutation}
  There exists an optimal solution for the discrete Kantorovich problem~\eqref{equation:discrete_kantorovich}, $\Pbf^{\sigma^*}$,
  which is a permutation matrix associated to an optimal permutation $\sigma^*\in\mathrm{Perm}(n)$ to the discrete Monge problem~\eqref{equation:discrete_monge}.
  Moreover, if $\set{\sigma^*_1,\dots,\sigma^*_M}\subset\mathrm{Perm}(n)$ is the set of optimal permutations to the discrete Monge problem~\eqref{equation:discrete_monge},
  the set of optimal solutions to the discrete Kantorovich problem~\eqref{equation:discrete_kantorovich} is the convex hull of $\set{\Pbf^{\sigma_1^*},\dots,\Pbf^{\sigma_M^*}}$.
\end{proposition}
Intuitively, \cref{proposition:permutation} says that we can recover a transport map from the optimal coupling.
This exact correspondence holds only when the source and target distributions are supported uniformly on the same number of points.
This stems from the fact that the set of extremal points of the Birkhoff polytope matches the set of permutation matrices~\citep{Birkhoff1946}.

\paragraph*{Cyclic monotonicity.}
We need to recall additional mathematical machinery from convex analysis.
We keep using a generic symmetric cost $c:\Rbb^d\times\Rbb^d\to\Rbb$, though we are mostly interested in the squared $L_2$ cost.
Refer to \citet[Section~5]{Villani2008} for more details.
\begin{definition}
  \label{definition:cm_set}
  A relation $\Gamma\subset\Rbb^d\times\Rbb^d$ is said to be a \emph{$c$-cyclically monotone} if, for any $m\in\Nbb$ and any family $\set{(\zbf_i,\ubf_i)}_{i\in[m]} \subset \Gamma$, the following inequality holds:
  \begin{equation}
    \label{equation:cm_condition}
    \sum_{i=1}^mc(\zbf_i,\ubf_i) \le \sum_{i=1}^mc(\zbf_i,\ubf_{i+1}),
  \end{equation}
  where we use the convention $\ubf_{m+1}=\ubf_1$.
  A coupling $\pi\in\mathscr{P}(\Rbb^d\times\Rbb^d)$ is said to be \emph{$c$-cyclically monotone} if $\supp(\pi)$ is $c$-cyclically monotone.
\end{definition}
Informally, a cyclically monotone coupling is locally optimal in the following sense:
The incurred transport cost (i.e., the left-hand side of \eqref{equation:cm_condition}) must grow if we locally perturb any of the coupled points.
\begin{definition}
  \label{definition:c-convex}
  A function $\Phi:\Rbb^d\to\Rbb\cup\set{\infty}$ is said to be \emph{$c$-convex} if $\setcomp{\zbf}{\Phi(\zbf)<\infty}\ne\emptyset$ and there exists $\zeta:\Rbb^d\to\Rbb\cup\set{\pm\infty}$ such that
  \begin{equation}
    \label{equation:c-convex}
    \forall \zbf\in\Rbb^d, \quad \Phi(\zbf) = \inf_{\ubf\in\Rbb^d}c(\zbf,\ubf) - \zeta(\ubf).
  \end{equation}
  The \emph{$c$-transform} of $\Phi$ defined as
  \[
    \forall \ubf\in\Rbb^d, \quad \Phi^c(\ubf) = \inf_{\zbf\in\Rbb^d} c(\zbf,\ubf) - \Phi(\zbf),
  \]
  and its \emph{$c$-subdifferential} is the following $c$-cyclically monotone relation:
  \[
    \partial_c\Phi \defeq \setcomp{(\zbf,\ubf) \in \Rbb^d\!\times\Rbb^d}{\Phi^c(\ubf) + \Phi(\zbf) = c(\zbf,\ubf)}.
  \]
  The $c$-subdifferential of $\Phi$ at point $\zbf$ is
\[
  \partial_c\Phi(\zbf)
  = \setcomp{ \ubf \in \Rbb^d }{ (\zbf,\ubf) \in \partial_c\Phi }.
\]

\end{definition}
Consider the cost $c_{\|\cdot\|}(\zbf,\ubf)=\|\zbf-\ubf\|_2^2$
and note that $c_{\|\cdot\|}$-cyclic monotonicity is equivalent to cyclic monotonicity with respect to $c_{\inpr{\cdot}{\cdot}}(\zbf,\ubf)\defeq-\inpr{\zbf}{\ubf}$~\citep{Smith1992}.
Then $c_{\inpr{\cdot}{\cdot}}$-subdifferential recover the usual subdifferential.
Thus, below, when we say cyclic monotonicity without specifying the cost, the cost is implicitly assumed to be $c_{\inpr{\cdot}{\cdot}}$.

The following proposition is a fundamental characterization of convex functions.
Therein, a relation $\Gamma$ is said to be \emph{maximally} monotone if no other monotone relation is strictly larger than it.
\begin{proposition}[{\citet[Theorem~3]{Rockafellar1966}}]
  \label{proposition:cm}
  A function $\Phi:\Rbb^d\to\Rbb$ is convex and $\setcomp{\zbf}{\Phi(\zbf)<\infty}\ne\emptyset$ if and only if there exists a maximal cyclically monotone relation $\Gamma\subset\Rbb^d\times\Rbb^d$ such that $\partial\Phi=\Gamma$.
  Moreover, $\Phi$ is unique up to an arbitrary additive constant.
\end{proposition}

\paragraph*{Monotonicity of transport maps.}
We recall the optimality characterization of the Kantorovich problem~\eqref{equation:kantorovich} via cyclic monotonicity.
\begin{proposition}[{\citet[Theorem~5.10]{Villani2008}}]
  \label{proposition:kantorovich}
  Let $c:\Rbb^d\times\Rbb^d\to\Rbb$ be l.s.c. such that $c(\zbf,\ubf) \ge a(\zbf) + b(\ubf)$ ($\forall(\zbf,\ubf)\in\Rbb^d\times\Rbb^d$)
  for some real-valued upper semicontinuous $a\in L^1(\mu)$ and $b\in L^1(\nu)$.%
  \footnote{
    The definition of lower semicontinuity can be found in \citet[Chapter~7]{Rockafellar1970}.
  }
  Assume that the optimal cost of \eqref{equation:kantorovich} is finite.
  Then the optimal coupling $\pi_*\in\Ucal(\mu,\nu)$ is $c$-cyclically monotone.
\end{proposition}
This result is useful because we can obtain a cyclically monotone coupling only by solving the linear program---%
if we consider the discrete Kantorovich problem~\eqref{equation:discrete_kantorovich}, standard solvers solve it with $O(n^3)$ time complexity~\citep{Bonneel2011SIGGRAPH},
which is far cheaper than testing all $O(2^n)$ subsets of a graph as in \cref{definition:cm_set}.
Under standard cost $c$, the optimal coupling is readily cyclically monotone with any regular measures.

\paragraph*{Barycentric map.}
In general, the optimal coupling $\pi_*$ does not give a single-valued transport map.
Hence, we need to define the following barycentric map~\citep[Definition~5.4.2]{Ambrosio2005}.
When $\pi$ admits the disintegration $\pi=\int\pi_{\zbf}\rd\mu(\zbf)$,
\[
  T_{\pi}(\zbf) = \int_{\Rbb^d} \ubf\rd\pi_{\zbf}(\ubf),
\]
which maps each source point $\zbf$ to a weighted barycenter of its neighbors in the target $\supp(\nu)$.

Beyond \cref{proposition:kantorovich}, we can establish a stronger result to ensure the existence of a convex potential that induces the optimal transport map.
This closes the loop to relate Kantorovich to Monge in general cases.
\begin{proposition}[{\citet{Brenier1991}}]
  \label{proposition:brenier}
  Consider the Monge problem~\eqref{equation:monge} with $c(\zbf,\ubf)=\|\zbf-\ubf\|_2^2$.
  If $\mu$ has a density with respect to the Lebesgue measure, then there exists a unique optimal transport map $T$
  that satisfies $T=\nabla\Phi$ for some differentiable convex function $\Phi:\Rbb^d\to\Rbb$ such that $(\nabla\Phi)_\sharp\mu=\nu$.
\end{proposition}
Make sure that this characterization is valid only when source $\mu$ is non-singular.
Otherwise, we must allow a transport plan to become multivalued $T_\pi(\zbf)\defeq\setcomp{\ubf}{(\zbf,\ubf) \in \supp(\pi)}$.
In parallel, $T_\pi$ is a \emph{sub}gradient of some convex potential~\citep{Peyre2024}.

\section{BRENIER IR}
\label{section:brenier_ir}
Consider the following CMIR (cyclically monotone isotonic regression) problem.
\begin{mdframed}[style=myframe]
  For a dataset $\set{(\zbf_i,\ybf_i)}_{i\in[n]}$, where $\zbf_i\in\Rbb^d$ is an input and $\ybf_i\in\set{0,1}^d$ is a one-hot encoded label, solve the following problem:
  \begin{align}
    \label{equation:cmir}
    & \min_{\hat\ybf_1,\dots,\hat\ybf_n\in\Rbb^d}\frac1n\sum_{i\in[n]}\|\ybf_i-\hat\ybf_i\|_2^2 \\
    \nonumber
    & \text{such that $\set{(\zbf_i,\hat\ybf_i)}_{i\in[n]}$ is cyclically monotone.}
  \end{align}
\end{mdframed}
The cyclically monotone constraint in CMIR can be reformulated as a discrete Kantorovich problem~\eqref{equation:discrete_kantorovich}, instead of directly unrolling \cref{definition:cm_set}.
\begin{enumerate}
  \vspace{-0.8\baselineskip} 

  \item Instead of optimize in $\set{\hat\ybf_i}_{i\in[n]}$, we solve the following \method{BrenierIR} problem.
  \begin{equation}
    \label{equation:brenier_ir}
    \begin{aligned}
      \min_{\ubf_1,\dots,\ubf_n\in\triangle^{d-1}} \; & \frac1n\|\Ybf - n\Pbf\Ubf\|_\frob^2 \\
      \text{subject to} \quad & \Pbf \in \argmin_{\Pbf \in \Bcal(n,n)} \inpr{\Cbf}{\Pbf},
    \end{aligned}
  \end{equation}
  where $\Ybf,\Ubf\in\Rbb^{n\times d}$ store $\set{\ybf_i}$ and $\set{\ubf_j}$ in rows, respectively.
  We specify the inner problem below.
  
  \item The inner problem of~\eqref{equation:brenier_ir} is the discrete Kantorovich problem~\eqref{equation:discrete_kantorovich} with source $\mu=\frac1n\sum_{i\in[n]}\delta_{\zbf_i}$, target $\nu=\frac1n\sum_{j\in[n]}\delta_{\ubf_j}$,
  and the cost $C_{ij}=\|\zbf_i-\ubf_j\|_2^2$.
  Denote its solution by $\Pbf^*\in\Bcal(n,n)$.

  \item With the barycentric map~\citep{Ferrandas2014SIAM}
  \begin{equation}
    \label{equation:barycentric_map}
    T_{\Pbf^*}(\zbf_i) = \frac{\sum_{j\in[n]}P^*_{ij}\ubf_j}{\sum_{j\in[n]}P^*_{ij}} = n\sum_{j\in[n]}P^*_{ij}\ubf_j,
  \end{equation}
  the prediction is given by $\hat\ybf_i=T_{\Pbf^*}(\zbf_i)$.
\end{enumerate}
Unlike the original CMIR, the inner problem of~\eqref{equation:brenier_ir} can be solved efficiently in practice via the implementation in \cref{figure:code} (described later).
The latent variable $\set{\ubf_j}$ can be viewed as the \emph{vector quantiles} of the observed variable $\set{\zbf_i}$~\citep{Carlier2016AOS}.
Thus, the transport associated to the coupling can be interpreted as a transformation from the raw input to the vector quantiles.
In most of the scenarios, the optimal coupling $\Pbf^*$ is associated with a permutation $\sigma^*$ (see \cref{proposition:permutation}).
Then the barycentric map $T_{\Pbf^*}$ is the optimal permutation from $\set{\zbf_i}$ to $\set{\ubf_j}$,
and $\zbf_i\mapsto\ubf_{\sigma^*(i)}$ has a cyclically monotone graph, thanks to \cref{proposition:kantorovich}.
Therefore, we have encoded the cyclic monotonicity constraint via the Kantorovich problem.

\begin{figure}
  \centering
  \includegraphics[width=\columnwidth]{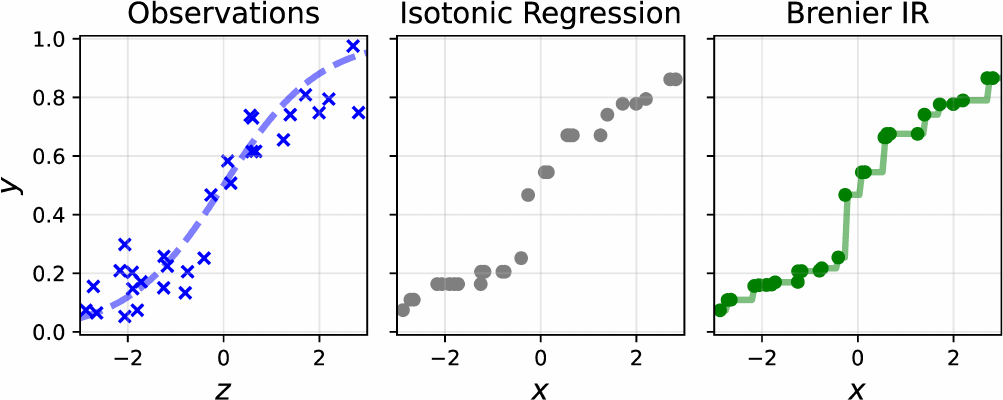}
  \caption{
    Comparison between IR and \method{BIR}.
    Observations are generated with $y=1/(1+\exp(-z))+\text{noise}$.
    The real line of \method{BIR} is computed via the Laguerre map over $\Rbb$.
  }
  \label{figure:compare}
\end{figure}

\paragraph*{Univariate case.}
Consider a 1D dataset $\set{(z_i,y_i)}\subset\Rbb\times\Rbb$ and latents $\ubf\defeq(u_i)_{i\in[n]}\subset\Rbb^n$ instead.
With the cost $C_{ij}=(z_i-u_j)^2$, the regression problem is
\begin{equation}
  \label{equation:brenier_ir_1d}
  \begin{aligned}
    \min_{\substack{\ubf,\hat\ybf\in\Rbb^n \\ \hat\ybf = n\Pbf\ubf}} \frac1n\sum_{i\in[n]}(y_i-\hat y_i)^2
    \text{~~subject to~~} \Pbf \in \Pcal(\ubf),
  \end{aligned}
\end{equation}
where $\Pcal(\ubf) \defeq \argmin\setcomp{\inpr{\Cbf}{\Pbf}}{\Pbf\in\Bcal(n,n)}$.
When $(z_i)_{i\in[n]}$ is increasing, we show that the set of $\hat\ybf$ spans the whole space of nondecreasing sequences.
\begin{theorem}
  \label{theorem:1d_monotonicity}
  Suppose $z_1<z_2<\dots<z_n$.
  Then the following equivalence holds:
\[
\begin{aligned}
  & \setcomp{ \hat{\ybf} = n \Pbf \ubf }
     { \Pbf \in \Pcal(\ubf)
     \text{ for some } \ubf \in \Rbb^n } \\
  & \hspace{50pt} =
    \setcomp{ \vbf \in \Rbb^n }{ v_1 \le v_2 \le \dots \le v_n }.
\end{aligned}
\]

\end{theorem}
As a corollary, \eqref{equation:brenier_ir_1d} is exactly equivalent to the standard IR.
\Cref{figure:compare} demonstrates univariate \method{BrenierIR} (with the implementation described at the end of \cref{section:brenier_ir}), which perfectly matches the PAV solution.

\paragraph*{User-specified number of bins.}
The original bi-level program allows the vector quantiles $\set{\ubf_j}$ to be all distinct,
i.e., it is possible that $|\supp(\nu)|=n$.
Consequently, the inner problem in \eqref{equation:brenier_ir} is subject to $O(n^3)$ time complexity.
To improve scalability in $n$, we implement an efficient variant with a hyperparameter $k\in\Nbb$
to set the target measure $\nu\defeq \frac1k\sum_{j\in[k]}\delta_{\ubf_j}$ such that $|\supp(\nu)|\le k(\ll n)$ is possible.
The efficient version of \eqref{equation:brenier_ir} will be referred to as $k$-\method{BrenierIR} with $k$ bins, which is defined as follows:
\begin{equation}
  \label{equation:brenier_ir_small}
  \begin{aligned}
    \min_{\ubf_1,\dots,\ubf_k\in\triangle^{d-1}} \; & \frac1n\|\Ybf - n\Pbf\Ubf\|_\frob^2 \\
    \text{subject to} \quad & \Pbf \in \argmin_{\Pbf \in \Bcal(n,k)} \inpr{\Cbf}{\Pbf},
  \end{aligned}
\end{equation}
where $\Ubf\in\Rbb^{k\times d}$ stores $\set{\ubf_j}_{j\in[k]}$ in each row.
The hyperparameter $k\in\Nbb$ represents the maximum number of bins and can be specified by a user.
Usually, $k$ must be at least greater than the number of classes $d$ for a valid prediction.
Given that classical IR trades off bias and variance adaptively~\citep{Zadrozny2002KDD}, the choice of $k$ may also affect the bias-variance tradeoff of \method{BrenierIR}.

Next, we formally state our first main theoretical contribution:
\method{BrenierIR} yields a cyclically monotone regression function, regardless of whether the optimal coupling is a permutation matrix or not.
All proofs can be found in \cref{section:proofs}.
\begin{theorem}[{Cyclic monotonicity of $T_{\Pbf^*}$}]
    \label{theorem:barycentric-CM}
    For fixed $\set{\ubf_j}_{j\in[k]}$ and the squared $L_2$ cost,
    let $\Pbf^*\in\Bcal(n,k)$ be the optimal transport from the observed inputs $\mu=\frac1n\sum_{i\in[n]}\delta_{\zbf_i}$ to $\nu=\frac1k\sum_{j\in[k]}\delta_{\ubf_j}$.
    Then the barycentric map $T_{\Pbf^*}$~\eqref{equation:barycentric_map} has a cyclically monotone graph.
\end{theorem}

\paragraph*{Test prediction.}
By solving \eqref{equation:brenier_ir_small}, we can obtain the barycentric map $\zbf_i\mapsto n\sum_{j\in[n]}P^*_{ij}\ubf_j$,
but this is applicable to training inputs only.
One way for test prediction is to re-compute \eqref{equation:brenier_ir_small} again after including test points, which is computational expensive.
A more principled approach is to leverage Brenier's theorem (see \cref{proposition:brenier}).
To do this, we need a non-singular source measure.
Consider the kernel density estimator
\begin{equation}
  \label{equation:kde}
  \tilde{\mu}(\zbf) \defeq \frac{1}{nh}\sum_{i\in[n]}K\left(\frac{\zbf - \zbf_i}{h}\right) \mathscr{L}^d(\zbf),
\end{equation}
where $\mathscr{L}^d$ is the Lebesgue measure in $\Rbb^d$ and $K$ is the standard normal density.
We can apply Brenier's theorem to non-singular $\tilde\mu$.

When source $\tilde{\mu}$ is continuous and target $\nu$ is discrete, the transport problem is called semi-discrete.
We introduce its solution based on \citet[Section~5]{Peyre2019}.
Given $\tilde\mu$ and $\nu=k^{-1}\sum_j\delta_{\ubf_j}$ with the squared $L_2$ cost, the Kantorovich problem~\eqref{equation:kantorovich} admits the dual formula:
\begin{equation}
  \label{equation:kantorovich_dual}
  \sup_{\gbf\in\Rbb^k}\Biggl\{\int_{\Rbb^d}g^c\rd\tilde\mu + \frac1k\inpr{\gbf}{\oneds_k}\Biggr\},
\end{equation}
where $\gbf\defeq[g(\ubf_1),\dots,g(\ubf_k)]^\top$ is the Lagrangian multiplier with respect to $\nu$,
$g:\Rbb^d\to\Rbb\cup\set{\infty}$ is the Kantorovich dual potential,
and $g^c$ is the $c$-transform of $g$.
Once \eqref{equation:kantorovich_dual} is solved, consider the \emph{Laguerre cells} associated to the Kantorovich potential $\gbf$:
\[
  \begin{aligned}
    &\Lbb_{\gbf}(\ubf_j) \\
    &\defeq \setcomp{ \zbf }{
      \forall j' \ne j,\, c(\zbf,\ubf_j) - g_j 
      \le c(\zbf,\ubf_{j'}) - g_{j'} }.
  \end{aligned}
\]

The Laguerre cells induce a disjoint partition such that $\Rbb^d = \bigsqcup_{j}\Lbb_{\gbf}(\ubf_j)$, which is a weighted extension of Voronoi tessellations~\citep{Aurenhammer1987}.
Given $\zbf\in\Rbb^d$ and an optimizer $g^*$ for \eqref{equation:kantorovich_dual},
define the Laguerre map $T_{\gbf^*}$ by $T_{\gbf^*}(\zbf)=\ubf_j$,
where $\Lbb_{\gbf^*}(\ubf_j)$ is the unique partition such that $\zbf\in\Lbb_{\gbf^*}(\ubf_j)$.
We note that $T_{\gbf^*}$ is piecewise constant by construction.
Therefore, the Laguerre map naturally extends the idea of the binning estimator to the multinomial case, and each quantile $\ubf_j$ serves as an individual bin level.
The map obtained by a measurable selection from Laguerre cells also exhibits the cyclically monotone property.
\begin{theorem}[Cyclic monotonicity of Laguerre map]
\label{theorem:laguerre-CM}
Let $\mu$ be an absolutely continuous probability measure on $\Rbb^d$ and 
$\nu=\sum_{j=1}^k b_j\delta_{\ubf_j}$ with $\bbf\in\triangle^{k-1}$. 
Let $\gbf^*\in\Rbb^k$ be an optimal dual potential to semi-discrete OT~\eqref{equation:kantorovich_dual} with the squared $L_2$ cost.
If $T_{\gbf^*}$ is a measurable selection such that $T_{\gbf^*}(\zbf)\in\setcomp{\ubf_j}{\zbf\in\Lbb_{\gbf^*}(\ubf_j)}$,
then $T_{\gbf^*}$ has a cyclically monotone graph.
\end{theorem}

As the Laguerre map $T_{\gbf^*}$ ends up with the gradient of Brenier's potential, we call our approach to CMIR as \emph{Brenier isotonic regression}.
The piecewise-constant nature is favorable for probability calibration, while the smoothness of $T$ should be improved for specific applications~\citep{Paty2020AISTATS,Peyre2024}.

\begin{figure}[t]
\begin{lstlisting}[language=Python]
 import numpy as np
 import ot
 from ot.utils.unif
 from scipy.optimize import minimize
 
 def obj(u_):
     u = u_.reshape(k, d)
     C = ot.dist(z, u)
     P = ot.emd(unif(n), unif(k), C)
     return np.sum((y - n*P@u)**2)/n
 
 res = minimize(obj, [..] method='SLSQP')
 u_opt = res.x.reshape(k, d)
\end{lstlisting}
\caption{\texttt{scipy} implementation of \method{BrenierIR}.
  Full code at \faGithub~\href{https://github.com/levelfour/Brenier_Isotonic_Regression}{\texttt{levelfour/Brenier\_Isotonic\_Regression}}.}
\label{figure:code}
\end{figure}

\paragraph*{Implementation.}
We describe a practical implementation of \eqref{equation:brenier_ir_small}.
A solution to \eqref{equation:brenier_ir} is recovered as its special case.
Generally speaking, bi-level programs including \eqref{equation:brenier_ir_small} is under a family of non-convex programs, called \emph{mathematical programming with equilibrium constraints} (MPEC)~\citep{Luo1996},
which has been an active research topic in the operations research community.
In our early experiments, an MPEC-related approach~\citep{Hillbrecht2024} did not perform ideally, while theoretically grounded.
Instead, we resort to simply estimating the derivative of the whole bi-level objective \eqref{equation:brenier_ir_small} in $\Ubf$ by the finite difference method.
This practically works surprisingly well, even though the objective \eqref{equation:brenier_ir_small} is not differentiable.%
\footnote{
  This is because a tiny perturbation to $\Ubf$ may bring the linear cost $\Cbf$ from one normal cone of $\Bcal(n,k)$ to another, where differentiability breaks down.
}
We summarize the implementation with \texttt{scipy}~\citep{2020SciPy-NMeth} and \texttt{pot} (Python OT)~\citep{POT} in \cref{figure:code}, where the main procedure consists of only a few lines.
With \texttt{jac=None} option, \texttt{scipy} uses the finite difference method to estimate the Jacobian.
For the optimization method, we used Sequential Quadratic Programming (SQP)~\citep{Nocedal2006}, specified by \texttt{method='SLSQP'},
which is an efficient second-order method and can deal with the linear constraints $\ubf_j\in\triangle^{d-1}$.

To compute the Laguerre cells, we need to solve semi-discrete OT~\eqref{equation:kantorovich_dual}.
It is typically solved by the averaged SGD~\citep{Genevay2016NeurIPS} over fresh samples from $\tilde\mu$~\eqref{equation:kde}.
While the time complexity is comparable to the primal Kantorovich problem, we adopt an additional heuristic to use the dual potential $\gbf^*$ returned by the primal Kantorovich problem without sampling fresh points.
This simple procedure suffices practically.
The Kantorovich potential can be obtained with \texttt{ot.emd},
which solves the primal Kantorovich problem, also known as the Earth Mover's Distance (EMD).

\section{APPLICATION 1: PROBABILITY~CALIBRATION}
\label{section:calibration}

\begin{table*}[t]
\centering
\caption{
  Recalibration results for MLP \textbf{(upper table)} and linear SVM \textbf{(lower table)}.
  Each number indicates the $L_1$ calibration error (lower is better) with averaging $10$ trials, and bold-faced if the recalibrator achieves the \underline{best or second best} performance or statistically indistinguishable from them by the Mann--Whitney $U$ test (significance level: $5$\%).
}
\label{table:calib}
\fontsize{9pt}{9pt}\selectfont
\begin{tabular}{lccccccccCCC}
\toprule
\multirow{2}{*}{\diagbox[width=10.5em]{Dataset}{Recalibrator}} & \multirow{2}{*}{---} & \multirow{2}{*}{\method{Bin}} & \multirow{2}{*}{\method{Dir}} & \multirow{2}{*}{\method{IRP}} & \multirow{2}{*}{\method{IR}} & \multirow{2}{*}{\method{MS}} & \multirow{2}{*}{\method{OI}} & \multirow{2}{*}{\method{TS}} & \multicolumn{3}{c}{\method{BrenierIR} \tiny{(Ours)}} \\
\cmidrule(lr){10-12}
&&&&&&&&& $k=15$ & $k=30$ & $k=50$ \\
\midrule
balance-scale & 0.244 & 0.184 & 0.108 & \textbf{0.068} & 0.139 & 0.171 & 0.140 & 0.177 & \textbf{0.061} & 0.070 & 0.084 \\
car & 0.063 & 0.050 & \textbf{0.030} & \textbf{0.031} & 0.034 & 0.037 & 0.132 & 0.036 & 0.045 & 0.040 & 0.042 \\
cleveland & 0.914 & 0.921 & 0.828 & \textbf{0.224} & 0.853 & 0.774 & 0.938 & 1.066 & \textbf{0.519} & 0.655 & 0.759 \\
dermatology & 0.178 & 0.187 & 0.153 & 0.204 & \textbf{0.139} & 0.167 & 0.798 & 0.163 & \textbf{0.122} & 0.159 & 0.170 \\
glass & 0.859 & 0.843 & 0.856 & \textbf{0.574} & 0.652 & 0.753 & 0.951 & 0.884 & \textbf{0.579} & 0.635 & 0.671 \\
vehicle & 0.294 & 0.300 & 0.208 & \textbf{0.103} & 0.199 & 0.310 & 0.474 & 0.298 & 0.177 & \textbf{0.145} & 0.202 \\
\bottomrule
\end{tabular} \\
\vspace{3pt}
\begin{tabular}{lccccccccCCC}
\toprule
\multirow{2}{*}{\diagbox[width=10.5em]{Dataset}{Recalibrator}} & \multirow{2}{*}{---} & \multirow{2}{*}{\method{Bin}} & \multirow{2}{*}{\method{Dir}} & \multirow{2}{*}{\method{IRP}} & \multirow{2}{*}{\method{IR}} & \multirow{2}{*}{\method{MS}} & \multirow{2}{*}{\method{OI}} & \multirow{2}{*}{\method{TS}} & \multicolumn{3}{c}{\method{BrenierIR} \tiny{(Ours)}} \\
\cmidrule(lr){10-12}
&&&&&&&&& $k=15$ & $k=30$ & $k=50$ \\
\midrule
balance-scale & 0.110 & 0.236 & 0.283 & \textbf{0.012} & 0.268 & 0.160 & 0.698 & 0.160 & \textbf{0.088} & 0.096 & 0.106 \\
car & \textbf{0.106} & 0.284 & 0.332 & 0.558 & 0.266 & 0.413 & 0.717 & 0.589 & \textbf{0.121} & 0.179 & 0.173 \\
cleveland & 0.784 & 0.855 & 0.732 & \textbf{0.255} & 0.905 & 0.655 & 0.746 & 0.896 & \textbf{0.573} & 0.725 & 0.730 \\
dermatology & 0.253 & 0.289 & 0.192 & 0.314 & \textbf{0.082} & 0.144 & 0.436 & 0.635 & 0.139 & 0.128 & \textbf{0.126} \\
glass & 0.831 & 0.795 & 0.846 & \textbf{0.044} & 0.649 & 0.711 & 0.780 & 0.905 & \textbf{0.647} & 0.685 & 0.799 \\
vehicle & 0.553 & 0.444 & 0.536 & \textbf{0.009} & 0.456 & 0.515 & 0.600 & 0.567 & \textbf{0.308} & 0.402 & 0.450 \\
\bottomrule
\end{tabular}
\end{table*}

\begin{figure*}[t]
  \centering
  \includegraphics[width=0.32\textwidth]{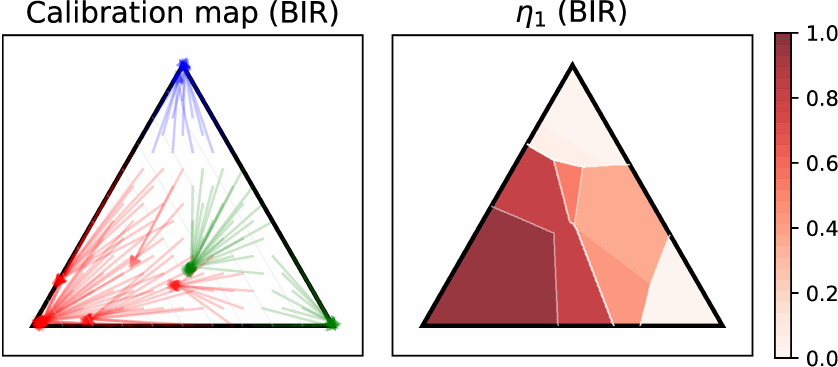}
  \includegraphics[width=0.32\textwidth]{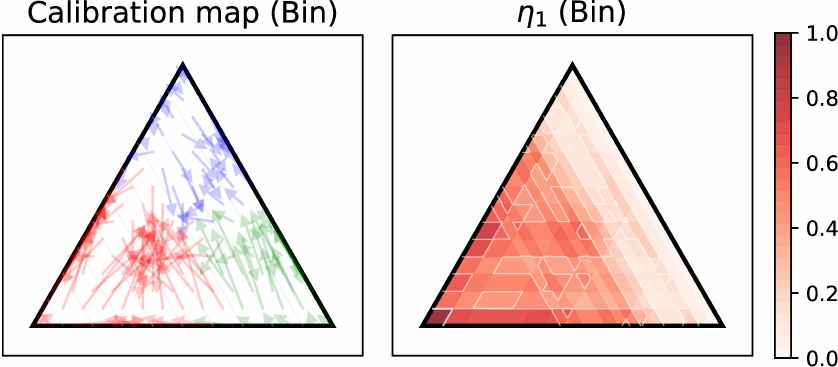}
  \includegraphics[width=0.32\textwidth]{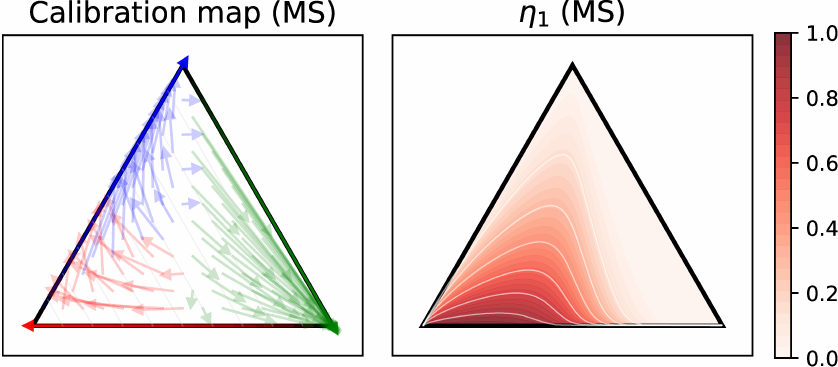}
  \caption{
    From left to right, \method{BrenierIR} ($k=50$), \method{Binning}, and \method{Matrix Scaling}.
    We show their estimated calibration maps as vector fields (each left) and contour plots of their first coordinate (each right).
  }
  \label{figure:cm}
\end{figure*}

As the first application, we use \method{BrenierIR} for calibrating multiclass classifiers.
Consider a probabilistic classifier $\hat{\pbf}:\Xcal \to \triangle^{d-1}$ that outputs class probabilities over $d$ classes.
It is called \emph{calibrated} if, for every prediction vector $\qbf\in\triangle^{d-1}$,
the conditional class proportions match the prediction:
\begin{equation}
  \label{equation:calibration}
  \forall i\in[d], \quad \Pr(Y=i\mid\hat\pbf(X)=\qbf) = q_i.
\end{equation}
When a decision maker acts in response to a forecaster, calibration ensures ``they speak the same language''~\citep{Foster2021}.
The measure equivalent to \eqref{equation:calibration} is the $L_1$ calibration error (CE)~\citep{Murphy1973}, while weaker measurements such as class-wise calibration error (cwCE) and expected calibration error (ECE)~\citep{Naeini2015AAAI} are also commonly used~\citep{Gruber2022NeurIPS}.

We are interested in recalibration (also known as post-hoc calibration):
Given a classifier $\hat\pbf$ and a small calibration set $\set{(\xbf_i,\ybf_i)}_{i\in[n]}\subset\Xcal\times\set{0,1}^d$,
estimate the canonical calibration map~\citep{Vaicenavicius2019AISTATS}:
\[
\eta(\qbf)=[\Pr(Y=i\mid \hat\pbf(X)=\qbf)]_{i\in[d]}.
\]
We estimate $\eta$ by \method{BrenierIR} applied on the pushforwarded calibration set $\set{(\hat\pbf(\xbf_i),\ybf_i)}_{i\in[n]}$.

\paragraph*{Benchmark.}
\Cref{table:calib} compares \method{BrenierIR} (shortened as \method{BIR}${}_k$) with baseline recalibrators:
\method{Bin} (OvR binning)~\citep{Zadrozny2001ICML},
\method{IR} (OvR isotonic regression)~\citep{Zadrozny2002KDD},
\method{MS} (matrix scaling, extending Platt's scaling)~\citep{Guo2017ICML},
\method{TS} (temperature scaling)~\citep{Guo2017ICML},
\method{Dir} (Dirichlet calibration)~\citep{Kull2019NeurIPS},
\method{OI} (order-invariant network)~\citep{Rahimi2020NeurIPS},
and \method{IRP} (iterative recursive partitioning)~\citep{Berta2024AISTATS}.
We recalibrated MLP ($D$-ReLU-$100$-ReLU-$d$ with the $L_2$ regularization strength $10^{-4}$, trained by Adam with $200$ epochs) and linear SVM (with the $L_2$ regularization strength $40$) and compared the recalibration performance by the $L_1$ calibration error (CE), which is the empirical estimate of $\E\|\Pr(Y|\hat\pbf(X)=\qbf) - \qbf\|_1$ averaged across $\qbf$ in the same bin $B\in\triangle^{d-1}$.
Here, the simplex binning is constructed by splitting each axis uniformly into $15$ intervals.
The further details of the baselines and implementations can be found in \cref{section:experiment_detail:setup}.
From our experimental results:
(i) \method{BrenierIR} consistently outperforms many baselines and performs comparably with \method{IRP}.
(ii) Even with relatively small $k$, \method{BrenierIR} matches IRP while scaling to larger numbers of classes, where \method{IRP} struggles.
In \cref{section:experiment_detail:time}, we show the average running time of \method{BrenierIR} and \method{IRP}, demonstrating that \method{BrenierIR} scales better.
Therefore, \method{BrenierIR} serves as a powerful recalibrator with a mild computational overhead.

We tested accuracy and classwise/confidence calibration error in \cref{section:experiment_details:benchmark}.
For classwise/confidence calibration error, the overall trend remains similar to \cref{table:calib}.
For accuracy, \method{BrenierIR} performs on par with the other baselines, while \method{IRP} is no longer competitive.

\paragraph*{Visualization.}
\Cref{figure:cm} visualizes the estimated calibration maps, recalibrated on a linear SVM trained with the balance-scale dataset~\citep{Shultz1994MLJ}.
\method{BrenierIR} has a strong binning effect compared with \method{Matrix Scaling}.
It also captures inter-class relations unlike the standard OvR \method{Binning}, whose contour levels are parallel to the simplex boundary.
As expected from the test prediction by the Laguerre map, the calibration map of \method{BrenierIR} concentrates on only a few points.
We visualize the calibration functions of the other baselines in \cref{section:experiment_detail:calibration_map},
where we can see qualitatively similar results between \method{BrenierIR} and \method{IRP}.

\begin{figure}
  \centering
  \includegraphics[width=0.49\columnwidth]{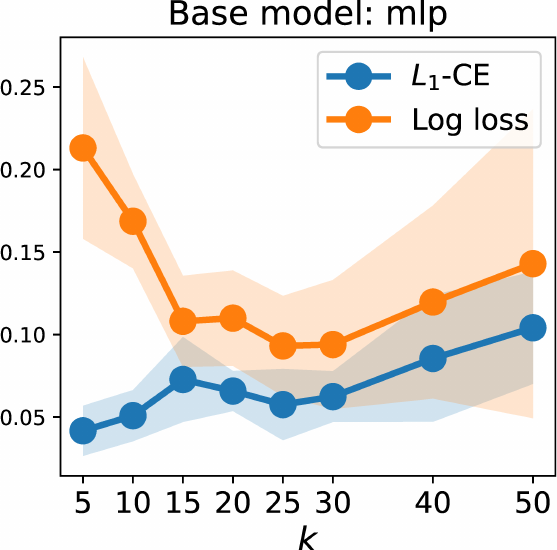}
  \includegraphics[width=0.49\columnwidth]{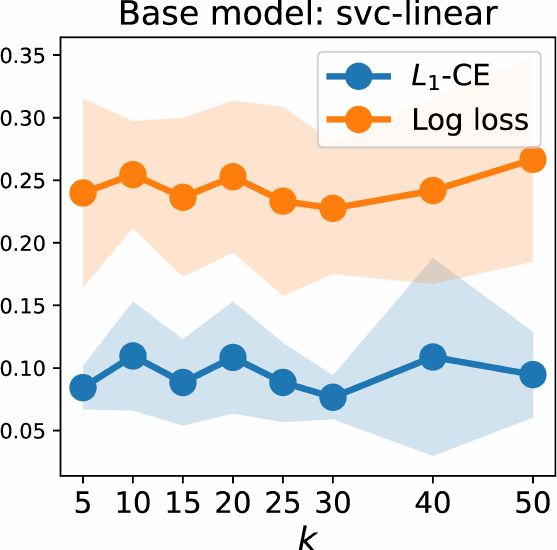}
  \caption{
    Each base model is recalibrated by \method{BIR} with different bin size $k$ (shown with the standard deviation).
  }
  \label{figure:trade}
\end{figure}

\begin{table*}[t]
\centering
\caption{
  Multiclass classification results.
  Each number indicates the accuracy and $L_1$ calibration error with averaging $10$ trials,
  and bold-faced if a method achieves the best performance or statistically indistinguishable from the best method by the Mann--Whitney $U$ test (significance level $5$\%).
}
\label{table:sim}
\fontsize{9pt}{8pt}\selectfont
\renewcommand{\arraystretch}{1.1}
\begin{tabular}{lcccCcccC}
\toprule
{Metric} & \multicolumn{4}{c}{Accuracy ($\uparrow$)} & \multicolumn{4}{c}{$L_1$ calibration error ($\downarrow$)} \\
\cmidrule(lr){1-1} \cmidrule(lr){2-5} \cmidrule(lr){6-9}
\diagbox[width=10em]{Dataset}{Method} & \method{Log} & \method{CLS} & \method{LT} & \makecell{\method{BIR}${}_{30}$ \\ \tiny{(Ours)}} & \method{Log} & \method{CLS} & \method{LT} & \makecell{\method{BIR}${}_{30}$ \\ \tiny{(Ours)}} \\
\hline
\addlinespace[2pt]
balance-scale & 0.860 & 0.682 & \textbf{0.910} & \textbf{0.895} & 0.227 & 0.507 & 0.170 & \textbf{0.117} \\
car & 0.657 & 0.509 & \textbf{0.964} & 0.808 & 0.831 & 0.899 & \textbf{0.086} & 0.120 \\
cleveland & \textbf{0.585} & 0.343 & \textbf{0.582} & \textbf{0.585} & 1.002 & 0.682 & 0.863 & \textbf{0.470} \\
glass & 0.630 & 0.328 & \textbf{0.684} & 0.595 & 1.030 & 0.682 & 0.793 & \textbf{0.608} \\
\bottomrule
\end{tabular}
\end{table*}

\paragraph*{Tradeoff in $k$.}
We trained MLP and linear SVM with the same setup on the balance-scale dataset, and recalibrated with \method{BIR} with different bin sizes $k$.
The trend shown in \cref{figure:trade} is rather clear in the MLP case: The log loss has a sweet spot in $k$.
CE exhibits a similar tradeoff for $k\in[15,50]$, but CE decreases for very small $k$.
This reflects predictions averaged over the simplex rather than genuine calibration.
The trend is less clear in the SVM case, but it calibrates slightly better around $k\simeq30$.
Overall, generalization analysis of recalibration is still underrepresented~\citep{Fujisawa2025ICML}, and further analysis is left for future work.

\section{APPLICATION 2: SINGLE-INDEX~MODELS}
\label{section:sim}

\begin{algorithm}[t]
  \caption{Brenier Single Index Model}
  \label{algorithm:brenier_sim}
  \KwIn{$T_{\max}$ total number of iters, $\Ubf_0$ randomly initialized}
  \KwOut{$\Wbf_{T_{\max}}$ learned index, $\Ubf_{T_{\max}}$ learned target support}
  \For{$0\le t\le T_{\max}-1$}{
    \Comment{Update linear predictor ($\Wbf$-step)}
    $\Wbf_{t+1} \gets \argmin_{\Wbf}J(\Wbf,\Ubf_t) + \frac{\lambda_W}{2}\|\Wbf\|_\frob^2$ \;
    \Comment{Update target measure ($\Ubf$-step)}
    $\Ubf_{t+1} \gets \argmin_{\Ubf}J(\Wbf_{t+1},\Ubf)$ \;
  }
\end{algorithm}

Let us revisit single-index models with multinomial outputs.
Given a dataset $\set{(\xbf_i,\ybf_i)}_{i\in[n]}$ for a covariate $\xbf_i\in\Xcal\subset\Rbb^D$ and a one-hot encoded label $\ybf_i\in\set{0,1}^d$,
we posit the multiclass classification model
\begin{equation}
  \label{equation:multiclass_model}
  \ybf_i\mid \xbf_i \sim \mathrm{Categorical}(\varphi(\Wbf^*\xbf_i)),
\end{equation}
where both the true index $\Wbf^*\in\Rbb^{d\times D}$ and the link function $\varphi:\Rbb^d\to\Rbb^d$ are known.
We assume the graph of $\varphi$ is cyclically monotone, as is usual to assume link monotonicity in SIMs.

\Cref{algorithm:brenier_sim} describes our approach to learn the model~\eqref{equation:multiclass_model} via \method{BrenierIR} ($k=30$).
We minimize the objective function
\[
  \begin{aligned}
    J(\Wbf,\Ubf) \defeq \; & \frac1n\|\Ybf-n\Pbf\Ubf\|_\frob^2 \quad \\
    \text{subject to} \;\; & \Pbf \in \argmin_{\Pbf \in \Bcal(n,k)}\inpr{\Cbf(\Wbf)}{\Pbf},
  \end{aligned}
\]
and the cost $\Cbf(\Wbf)$ is $C_{ij}\defeq\|\Wbf\xbf_i-\ubf_j\|_2^2$.
Both $\Wbf$- and $\Ubf$-steps can be implemented via SQP with the finite difference method, as in \cref{figure:code}.
This alternating procedure is reminiscent of the calibrated least squares~\citep[Algorithm~2]{Agarwal2014ICML}, which alternates the residual fitting and the link fitting steps.

We compared with the following baselines:
\method{Log} is the standard multinomial logistic regression with the $L_2$ regularization strength $1.0$ (optimized by L-BFGS);
\method{CLS} is the calibrated least squares~\citep{Agarwal2014ICML}.
We ran $100$ and $2000$ iters for \method{BIR} and \method{CLS}, respectively.
\method{CLS} requires introducing bases of the monotone link, for which we used $\set{z,z^2,z^3}$ as the original paper suggests.
\method{LT} is the LegendreTron~\citep{Lam2023ICML}, which learns the inverse link via ICNNs together with the linear predictor.
The baseline details are again deferred to \cref{section:experiment_detail}.

\Cref{table:sim} demonstrates the results of multiclass classification on each dataset.
We measured the accuracy and $L_1$ calibration error of each learned SIM.
\method{BIR} consistently outperforms \method{CLS}, which also models SIMs explicitly but parametrically.
This admits how a nonparametric approach performs nicely in practice.
While \method{BIR} greatly improves calibration over \method{LT}, classification performance of \method{BIR} is not ideal.
Therefore, we recommend practitioners to use \method{BIR} only for recalibration for now and leave developing better classification algorithms for future work.

\section{DISCUSSION}

\paragraph*{\method{BrenierIR} as adaptive binning.}
\method{Binning} and \method{IR} are the two most common recalibrators~\citep{Zadrozny2002KDD}.
\method{Binning} is a simple nonparametric approach yet yielding a nice bias-variance tradeoff, by averaging confidence prediction falling into the same bin~\citep{Zadrozny2001ICML}.
\method{IR} can be viewed as an adaptive binning such that the bin boundaries are determined depending on the base model (detailed in \citet{Guo2017ICML}).
While their OvR extensions are insensitive to inter-class correlation---in fact, all contour levels are parallel to the simplex boundaries in \cref{figure:cm}---%
\method{BrenierIR} captures inter-class correlation and yields class-adaptive simplex binning, as seen in \cref{figure:cm}.
Such class-adaptive simplex binning is possible by \method{IRP}~\citep{Berta2024AISTATS}, but it needs to sweep the uniform grid points over the whole probability simplex, which is computationally demanding for large $d$.
Thus, \method{BrenierIR} is a tractable and class-adaptive binning.

\paragraph*{Future work.}
Last but not least, we recognize that the computational cost of \method{BrenierIR} remains a bottleneck.
Even the inner OT requires $O(n^3)$ time complexity, and the total time complexity of the bi-level program is opaque.
Given fairly practical performance of \method{BrenierIR}, further investigation from the optimization perspective is an interesting open direction.

\section*{ACKNOWLEDGMENT}
HB is supported by JST PRESTO JPMJPR24K6.
YW acknowledges support from NSF CISE CRII 2451714.
The projected was initiated during YW's research visit to ISM.

\bibliographystyle{plainnat}
\bibliography{reference}

\clearpage
\appendix
\onecolumn

\section{PROOFS}
\label{section:proofs}

\subsection{Proof of Theorem~\ref{theorem:1d_monotonicity}}
\label{section:proof:1d_monotonicity}
Before proceeding with the proof, we show the following claim:
Under the assumption of $z_1<z_2<\dots<z_n$, if a permutation matrix $\Pbf^{\sigma^*}$ with a permutation $\sigma^*\in\mathrm{Perm}(n)$ is optimal to the Kantorovich problem $\Pcal(\ubf)=\argmin\setcomp{\inpr{\Cbf}{\Pbf}}{\Pbf\in\Bcal(n,n)}$ with the cost $C_{ij}=(z_i-u_j)^2$,
then $\sigma^*$ must be \emph{order-preserving}, that is,
$u_{\sigma^*(1)}\le u_{\sigma^*(2)}\le \dots \le u_{\sigma^*(n)}$.
This can be seen based on the following inequality:
\[ 
  \text{For any $i < k$ and $j < l$ such that $u_j \le u_l$,} \quad
  (C_{il} + C_{kj}) - (C_{ij} + C_{kl}) = 2(z_k-z_i)(u_l-u_j) \ge 0.
\]
If a permutation $\sigma\in\mathrm{Perm}(n)$ has a ``crossing'' quadruple $(i,j,k,l)$ such that $j=\sigma(i)>\sigma(k)=l$ but $i<k$, the ``uncrossed'' permutation
\[
  \sigma'(i_0) = \begin{cases}
    \sigma(i) & \text{if $i_0=k$,} \\
    \sigma(k) & \text{if $i_0=i$,} \\
    \sigma(i_0) & \text{otherwise}
  \end{cases}
\]
yields a lower (or nonincreasing) transport cost than that of $\sigma$.
Therefore, any permutation can be made order-preserving by uncrossing all crossing quadruples,
and the Kantorovich problem is optimized at an order-preserving permutation.

\paragraph*{The forward inclusion ($\subset$).}
Suppose that $\hat\ybf\in\Rbb^n$ can be written as $\hat\ybf=n\Pbf\ubf$ for some $\ubf\in\Rbb^n$ and $\Pbf\in\Pcal(\ubf)$.
By \cref{proposition:permutation}, there exists $M\ge 1$ and $\set{\sigma_1^*,\dots,\sigma_M^*}\subset\mathrm{Perm}(n)$ such that
$\Pbf$ can be represented as a convex combination of permutation matrices $\set{\Pbf^{\sigma_i^*}}_{i\in[M]}$,
where each $\Pbf^{\sigma_i^*}$ is optimal to $\Pcal(\ubf)$.
Let us write this convex combination by
\[
  \Pbf = \sum_{i\in[M]}\lambda_i\Pbf^{\sigma_i^*}
\]
for some $(\lambda_i)_{i\in[M]}\in[0,1]^M$ such that $\lambda_1+\dots+\lambda_M=1$.
Now we have
\[
  \hat y_j = n[\Pbf\ubf]_j
  = n\sum_{i\in[M]}\lambda_i\cdot\frac1n\cdot u_{\sigma_i^*(j)}
  = \sum_{i\in[M]}\lambda_iu_{\sigma_i^*(j)}
  \quad \text{for any $j\in[n]$},
\]
and for $j < j'$,
\[
  \hat y_j = \sum_{i\in[M]}\lambda_iu_{\sigma_i^*(j)}
  \le \sum_{i\in[M]}\lambda_iu_{\sigma_i^*(j')}
  = \hat y_{j'}
\]
because the optimality of each $\Pbf^{\sigma_i^*}$ indicates the order-preserving property of $\sigma_i^*$.
Thus, $\hat y_1\le \hat y_2 \le \dots \le \hat y_n$.

\paragraph*{The reverse inclusion ($\supset$).}
Fix a fixed nondecreasing sequence $\vbf\in\Rbb^n$.
We choose $\ubf=\vbf$ and $\Pbf=\frac1n\Ibf_n$, where $\Ibf_n\in\Rbb^{n\times n}$ is the identity matrix.
This $\Pbf$ corresponds to the identity permutation and is obviously order-preserving.
Hence, $\Pbf$ is optimal to the Kantorovich problem and
\[
  \hat y_j = n[\Pbf\ubf]_j = n\cdot\frac1n v_j = v_j.
\]
This proves the reverse inclusion.
\qed

\subsection{Proof of Theorem~\ref{theorem:barycentric-CM}}
\label{section:proof:barycentric-CM}
First, we recapitulate the setup.
Let $\mu=\frac1n\sum_{i=1}^n \delta_{\zbf_i}$ and $\nu=\frac1k\sum_{j=1}^k \delta_{\ubf_j}$, and $\Pbf^*\in\Bcal(n,k)$ be an optimal solution to the following problem:
\[
\argmin_{\Pbf\in\Bcal(n,k)} \inpr{\Cbf}{\Pbf},
\text{~~where~~} C_{ij}=\frac12\|\zbf_i-\ubf_j\|_2^2
\text{~~and~~} \Bcal(n,k)=\setcomp{\Pbf\in\Rbb_{\ge0}^{n\times k}}{\Pbf\oneds_k=\frac1n, \Pbf^\top\oneds_n = \frac1k}.
\]
Let $\fbf^*\in\Rbb^n$ and $\gbf^*\in\Rbb^k$ be the optimal potentials to the Kantorovich dual problem, then we have
\begin{equation}
  \label{eq:dual}
  \begin{cases}
    f_i^* + g_j^* \le \frac12\|\zbf_i-\ubf_j\|_2^2 \\
    P_{ij}^* > 0 \implies f_i^* + g_j^* = \frac12\|\zbf_i-\ubf_j\|_2^2
  \end{cases}
\end{equation}
for any $(i,j)\in[n]\times[k]$ due to the complementary slackness (see \citet[Section~2.5]{Peyre2019}).
The barycentric map $T_{\Pbf^*}:\set{\zbf_1,\dots,\zbf_n} \to \convhull\set{\ubf_1,\dots,\ubf_k}$ is given by
\[
  T_{\Pbf^*}(\zbf_i) = n\sum_{j=1}^kP^*_{ij}\ubf_j,
  \quad \text{for $i\in[n]$}
\]
so that $n\sum_{j=1}^kP^*_{ij}=1$ for each $i\in[n]$.

Now we show that there exists a proper convex function $\Phi:\Rbb^d\to\Rbb$ such that $T_{\Pbf^*}(\zbf_i)\in\partial\Phi(\zbf_i)$ for every $i\in[n]$.
Once we establish this fact, the graph of the barycentric map $\setcomp{(\zbf_i,T_{\Pbf^*}(\zbf_i))}{i\in[n]}$ immediately turns out to be cyclically monotone (thanks to \citet[Theorem~24.8]{Rockafellar1970}).%
\footnote{
  Note that \citet[Theorem~24.8]{Rockafellar1970} establishes the equivalence between the existence of a closed proper convex function and cyclic monotonicity of (any subset of) its subgradient.
  Precisely speaking, this requires that the subgradient subset is defined over the entire $\Rbb^d$, but the barycentric map $T_{\Pbf^*}$ for discrete OT is only defined over $\set{\zbf_1,\dots,\zbf_n}$.
  Nevertheless, this does not cause any issue because we only use the necessity of \citet[Theorem~24.8]{Rockafellar1970}, where we do not need to take care of whether $T_{\Pbf^*}$ is defined over the entire $\Rbb^d$ or not.
}
To show this, we construct the convex function as follows:
\begin{equation}
  \label{eq:Phi}
  \Phi(\zbf)\defeq\max_{1\le j\le k}\set{\inpr{\zbf}{\ubf_j}-\frac12\|\ubf_j\|_2^2+g_j^*}.
\end{equation}
First, the maximum of $\Phi(\zbf_i)$ is attained at the ``active'' $\ubf_j$, namely, any $\ubf_j$ with $P_{ij}^*>0$.
Indeed, for any $(i,j)\in[n]\times[k]$, the complementary slackness~\eqref{eq:dual} implies
\[
  f_i^* + g_j^* \le \frac12\|\zbf_i-\ubf_j\|_2^2 = \frac12\|\zbf_i\|_2^2 + \frac12\|\ubf_j\|_2^2 - \inpr{\zbf_i}{\ubf_j},
\]
which is rearranged as follows:
\[
  \inpr{\zbf_i}{\ubf_j} - \frac12\|\ubf_j\|_2^2 + g_j^* \le \frac12\|\zbf_i\|_2^2 - f_i^*.
\]
By taking the maximum over $j\in[k]$, this gives
\[
  \Phi(\zbf_i) \le \frac12\|\zbf_i\|_2^2 - f_i^*,
\]
and its equality holds when $P_{ij}^*>0$.
Thus, we have $\Phi(\zbf_i)=\inpr{\zbf_i}{\ubf_j}-\frac12\|\ubf_j\|_2^2+g_j^*$ for active $\ubf_j$, though active $\ubf_j$ for $\zbf_i$ may not be unique.
Regardless of the uniqueness, we have $\ubf_j\in\partial\Phi(\zbf_i)$, which implies
\[
  \setcomp{\ubf_j}{P_{ij}^*>0} \subset \partial\Phi(\zbf_i).
\]
Lastly, the convex combination $T_{\Pbf^*}(\zbf_i)=n\sum_jP_{ij}^*\ubf_j$ must lie in $\partial\Phi(\zbf_i)$ because $\partial\Phi(\zbf_i)$ is a closed convex set.
Therefore, $T_{\Pbf^*}(\zbf_i)\in\partial\Phi(\zbf_i)$ for every $i\in[n]$, which is the desired argument.
\qed

\subsection{Proof of Theorem~\ref{theorem:laguerre-CM}}
\label{section:proof:laguerre-CM}
First, we recapitulate the setup.
Let $\mu$ be an absolutely continuous probability measure on $\Rbb^d$ and $\nu=\sum_{j=1}^kb_j\delta_{\ubf_j}$ for $\bbf\in\triangle^{k-1}$,
and $\gbf^*\in\Rbb^k$ be the optimal potential to the semi-discrete OT problem:
\[
  \sup_{\gbf\in\Rbb^k}\setcomp{\int_{\Rbb^d}g^c\rd\mu + \sum_{j=1}^kb_jg_j}{g^c(\zbf)+g_j \le \|\zbf-\ubf_j\|_2^2 \text{~~for any $\zbf\in\Rbb^d$ and $j\in[k]$}},
\]
where $g^c$ is the $c$-transform of $g:\Rbb^d\to\Rbb\cup\set{\infty}$ associated with the squared $L_2$ cost:
\[
  g^c(\zbf) \defeq \min_{1\le j\le k}\set{\|\zbf-\ubf_j\|_2^2-g_j}
  \quad \text{for $\zbf\in\Rbb^d$.}
\]
The Laguerre map $T_{\gbf^*}:\Rbb^d\to\Rbb^d$ is given by 
\[
  \begin{aligned}
    T_{\gbf^*}(\zbf) &\in \setcomp{\ubf_j}{\|\zbf-\ubf_j\|_2^2-g_j^* \le \|\zbf-\ubf_{j'}\|_2^2-g_{j'}^* \text{~~for $j'\ne j$}} \\
    &= \setcomp{\ubf_j}{\inpr{\zbf}{\ubf_j}-\frac12\|\ubf_j\|_2^2 + \frac12g_j^* \ge \inpr{\zbf}{\ubf_{j'}} - \frac12\|\ubf_{j'}\|_2^2 + \frac12g_{j'}^* \text{~~for $j'\ne j$}},
  \end{aligned}
\]
where the tie can broken arbitrarily as long as $T_{\gbf^*}$ is measurable.

Construct a closed proper convex function
\[
  \Phi(\zbf) \defeq \max_{1\le j\le k}\set{\inpr{\zbf}{\ubf_j} - \frac12\|\ubf_j\|_2^2 + g_j^*}.
\]
By following the same strategy as the proof of \cref{theorem:barycentric-CM}, we can show that $T_{\gbf^*}(\zbf)\in\partial\Phi(\zbf)$,
which implies that $T_{\gbf^*}$ has a cyclically monotone graph.
\qed

\section{MONOTONICITY NOTIONS}
\label{section:monotonicity}
In this section, we briefly discuss a different notion of monotonicity, called the \emph{intra-order preserving} (IOP) property.
This property is recently used by \citet{Rahimi2020NeurIPS} for recalibration.
Specifically, they learn a neural network with the IOP property for recalibration
because it is expected to be a reasonable property as a recalibrator but does not loses the expressivity too much.
Here, we discuss that cyclic monotonicity induces a slightly weaker notion of the IOP property.

\begin{definition}
A function $f:\Rbb^d\to\Rbb^d$ is said to be \emph{weakly intra-order preserving}
if for every $\xbf\in\Rbb^d$ and every pair $(i,j)\in[d]^2$, we have
$x_i \le x_j \implies f_i(\xbf)\le f_j(\xbf)$.
Equivalently, $f$ is weakly intra-order preserving if we have
$(x_i-x_j)(f_i(\xbf)-f_j(\xbf)) \ge 0$ for any $\xbf\in\Rbb^d$ and $(i,j)\in[d]^2$.
\end{definition}

\begin{definition}
A function $f:\Rbb^d\to\Rbb^d$ is said to be \emph{permutation equivariant} if we have
$f(\Pbf\xbf)=\Pbf f(\xbf)$ for every permutation matrix $\Pbf$ and every $\xbf\in\Rbb^d$.
\end{definition}

\begin{proposition}
\label{prop:weak-IOP}
Let $f:\Rbb^d\to\Rbb^d$ be permutation equivariant and cyclically monotone.
Then $f$ is weakly intra-order preserving.
In particular, if $x_i=x_j$ then $f_i(\xbf)=f_j(\xbf)$.
\end{proposition}

\begin{proof}
Since $f$ is cyclically monotone, \cref{proposition:cm} yields that
for any $\xbf,\ybf\in\Rbb^d$,
\[
  \inpr{f(\xbf)-f(\ybf)}{\xbf-\ybf}\ge 0.
\]
For a fixed $\xbf\in\Rbb^d$ and indices $i,j\in[n]$ such that $i\ne j$.
Let $\Pbf\in\Rbb^{d\times d}$ be the following permutation matrix:
\[
  P_{i_0j_0} = 
  \begin{cases}
    1 & \text{if $(i_0,j_0)\in\set{(i,j),(j,i)}\cup\setcomp{(i_0,i_0)}{i_0\ne i,j}$,} \\
    0 & \text{otherwise,}
  \end{cases}
\]
which swaps the $i$-th and $j$-th elements.
If we set $\xbf'\defeq\Pbf\xbf$, then we have
\[
  0 \le \inpr{f(\xbf)-f(\xbf')}{\xbf-\xbf'}
  = \inpr{f(\xbf)-\Pbf f(\xbf)}{\xbf-\Pbf\xbf}
  = 2(f_i(\xbf)-f_j(\xbf))(x_i-x_j),
\]
which is the weakly IOP property.
If we have $x_i=x_j$ additionally, then $\xbf'=\xbf$ and $f(\xbf)=\Pbf f(\xbf)$,
which implies $f_i(\xbf)=f_j(\xbf)$.
\end{proof}

\section{DETAILED EXPERIMENTS}
\label{section:experiment_detail}

\subsection{Full setup}
\label{section:experiment_detail:setup}

\paragraph{Datasets.}
The full list of datasets we used in the experiments is shown in \cref{table:dataset}.
All datasets are available on OpenML~\citep{OpenML2025}.
For each dataset, we created a random train-test split with the ratio $8$ to $2$ first.
For calibration experiments (\cref{section:calibration}), we applied $3$-fold stratified cross validation further to the train split, split into train and calibration splits.
The base models were trained with the train splits,
while the recalibrators were trained with calibration splits and their hyperparameters (if any) were chosen based on the train splits with the validation log loss.
Then, each recalibrator trained with an individual calibration split were averaged upon recalibration, forming model averaging.
This practice is common in literature~\citep{Kull2019NeurIPS}.

\begin{table}
  \centering
  \caption{Dataset details.}
  \label{table:dataset}
  \begin{tabular}{llll}
    \toprule
    Dataset & Sample size & \# of features & \# of classes \\
    \midrule
    balance-scale & 625 & 4 & 3 \\
    car & 1728 & 6 & 4 \\
    cleveland & 297 & 13 & 5 \\
    dermatology & 358 & 34 & 6 \\
    glass & 214 & 9 & 6 \\
    segment & 2310 & 19 & 7 \\
    vehicle & 846 & 18 & 4 \\
    yeast & 1484 & 8 & 10 \\
    \bottomrule
  \end{tabular}
\end{table}

\paragraph{Recalibrators.}
We describe the implementation of each baseline recalibrator (used in \cref{section:calibration}).
Except for \method{Bin} and \method{IR}, all the other baselines are multiclass recalibrators.
\begin{itemize}
  \item \method{Bin} (OvR binning) \citep{Zadrozny2001ICML}:
  We used the implementation by \citet{Kull2019NeurIPS}.%
  \footnote{
    \url{https://github.com/dirichletcal/experiments_neurips}
  }
  The number of bins of the base OvR binning estimator was fixed to $15$,
  and each bin was created uniformly over $[0,1]$.
  Then each base estimator was aggregated in the OvR fashion, turned into a multiclass recalibrator.

  \item \method{IR} (OvR isotonic regression) \citep{Zadrozny2002KDD}:
  We used \texttt{scikit-learn} implementation of binary isotonic regression~\citep{scikit-learn}.
  Each binary isotonic regressor was aggregated in the OvR fashion, turned into a multiclass recalibrator.

  \item \method{MS} (matrix scaling) \citep{Guo2017ICML}:
  Let $\zbf_i\in\Rbb^d$ be a vector of multinomial logits.
  Then matrix scaling estimates the confidence of class $i$ by $\hat p_i=\max_{c\in[d]}\mathrm{softmax}_c(\Wbf\zbf_i+\bbf)$,
  where $\mathrm{softmax}_c$ is the $c$-th coordinate of the softmax output, and $\Wbf\in\Rbb^{d\times d}$ and $\bbf\in\Rbb^d$ are optimized with respect to the NLL.
  Later, \citet{Kull2019NeurIPS} proposed to add the $L_2$ regularization to make $\Wbf$ stay close to the identity matrix $\Ibf$.
  We used the implementation by \citet{Kull2019NeurIPS}, which optimizes the NLL by L-BFGS.
  The regularization hyperparameter was chosen from $\set{10^2, 10^0, 10^{-2}, 10^{-4}}$ by the aforementioned cross-validation on the calibration splits.

  \item \method{TS} (temperature scaling) \citep{Guo2017ICML}:
  It can be seen as a more restricted version of \method{MS}, such as $\hat p_i = \max_{c\in[d]}\mathrm{softmax}_c(\zbf_i/\tau)$,
  where $\tau>0$ is the temperature parameter optimized with respect to the NLL.
  We used the implementation by \citet{Kull2019NeurIPS}.

  \item \method{Dir} (Dirichlet calibration) \citep{Kull2019NeurIPS}:
  It models the calibration map $\eta(\qbf)=[\Pr(Y=i\mid \hat\pbf(X)=\qbf)]_{i\in[d]}$ by the Dirichlet distribution.
  The linear parametrization of Dirichlet recalibrator is $\eta(\qbf)=\mathrm{softmax}(\Wbf\ln\qbf+\bbf)$, where $\qbf\in\triangle^{d-1}$ is the input of a calibration map, or can be seen as a probabilistic output of the base model.
  The parameters $\Wbf\in\Rbb^{d\times d}$ and $\bbf\in\Rbb^d$ are optimized with respect to the NLL.
  This is very similar to \method{MS}, with the only difference replacing the logit input $\zbf$ with $\ln\qbf$.
  We used the implementation by \citet{Kull2019NeurIPS}, which optimizes the NLL by L-BFGS.
  In addition, we used the same $L_2$ regularization as \method{MS}, with the regularization hyperparameter chosen from $\set{10^2, 10^0, 10^{-2}, 10^{-4}}$.

  \item \method{OI} (order-invariant network) \citep{Rahimi2020NeurIPS}:
  To boost the expressivity of recalibrators with minimal inductive bias, order-invariant networks were proposed.
  That is, a recalibrator must be equivariant to a permutation applied to the input logit vector.
  This can be implemented by $\fbf(\zbf)=S(\zbf)^{-1}\Ubf\wbf(\Sbf(\zbf)\zbf)$, where $S(\zbf)\in\Rbb^{d\times d}$ the permutation matrix yielding the descending order of $\zbf\in\Rbb^d$, $\Ubf\in\Rbb^{d\times d}$ is an upper-triangular matrix of ones, and $\wbf:\Rbb^d\to \Rbb^d$ is subject to some constraints relevant to the relative order of $\zbf$-components (see the original paper for the detail).
  The function $\wbf$ has a small degree-of-freedoms, and the original paper proposed to model it by a two-layer multilayer perceptron.
  Based on the official implementation,\footnote{\url{https://github.com/AmirooR/IntraOrderPreservingCalibration/}} we reimplemented the whole recalibrator by optimizing the NLL with L-BFGS.
  Upon the recalibrator training, we applied the $L_2$ regularization to all perceptron parameters, and the regularization hyperparameter was chosen from $\set{10^2, 10^0, 10^{-2}, 10^{-4}}$.

  \item \method{IRP} (iterative recursive splitting) \citep{Berta2024AISTATS}:
  This is a direct extension of \method{Bin} to the probability simplex with higher dimensions.
  The basic procedure is to recursively find a point to split the probability simplex, and eventually creating bins over the simplex.
  After terminating binning, the recalibrator is defined as the mean prediction over each bin.
  Specifically, by creating grid points over the entire probability simplex, the algorithm sweeps the grid points to find out the point where the entropy criterion is maximized after the resulting simplex splitting.
  Based on the official implementation,\footnote{\url{https://github.com/eugeneberta/Calibration-ROC-IR/}} we reimplemented the algorithm by fixing the grid resolution (along each dimension) to $k_0=10$.
  Although this is not satisfactory, the entire number of grid points becomes $O(k_0^d)$ with $d$ classes, which is computationally intractable with too large $k_0$.
\end{itemize}

\paragraph{Single-index models.}
We describe the implementation of each baseline (multinomial) single-index models (used in \cref{section:sim}).
Note that we intentionally focus on multiclass classifiers based on multinomial GLMs among numerous candidate baselines.
All of the baselines described here model a multinomial inverse link function (more or less).

\begin{itemize}
  \item \method{Log} (multinomial logistic regression):
  We used \texttt{scikit-learn}~\citep{scikit-learn} to implement with the default hyperparameters.
  The loss function is optimized by L-BFGS.
  The multinomial link function is fixed to the softmax function.

  \item \method{CLS} (calibrated least squares) \citep{Agarwal2014ICML}:
  This is an alternative algorithm to fit a linear predictor to the residual and calibrate the linear predictor to the observed outcomes by a cyclically monotone link.
  The monotone link is modeled by an affine combination of bases.
  As suggested by the original paper, we used $\set{z,z^2,z^3}$ for the bases.
  Both steps can be executed analytically by solving least squares.
  We alternated these two steps for $2000$ steps, with the $L_2$ regularization for the latter step (the regularization strength was set to $10^{-6}$).

  \item \method{LT} (LegendreTron) \citep{Lam2023ICML}:
  This models the multinomial inverse link function via input-convex neural networks (ICNNs)~\citep{Amos2017ICML}, which is applied on top of a linear predictor.
  We can immediately obtain a monotone link by taking the gradient of ICNNs with respect to inputs.
  The original paper suggests ICNN with $2$ blocks, $4$ hidden units, and $4$ layers.
  We used the official implementation,\footnote{\url{https://github.com/khflam/legendretron/}} without changing the default hyperparameter setups, and trained for $240$ epochs.
\end{itemize}

\begin{figure}[t]
  \centering
  \includegraphics[width=\columnwidth]{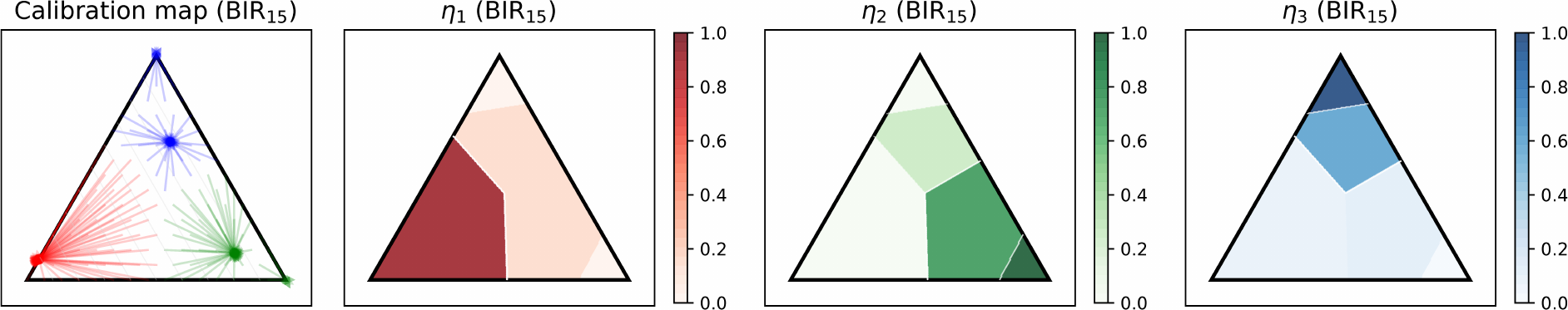} \\
  \includegraphics[width=\columnwidth]{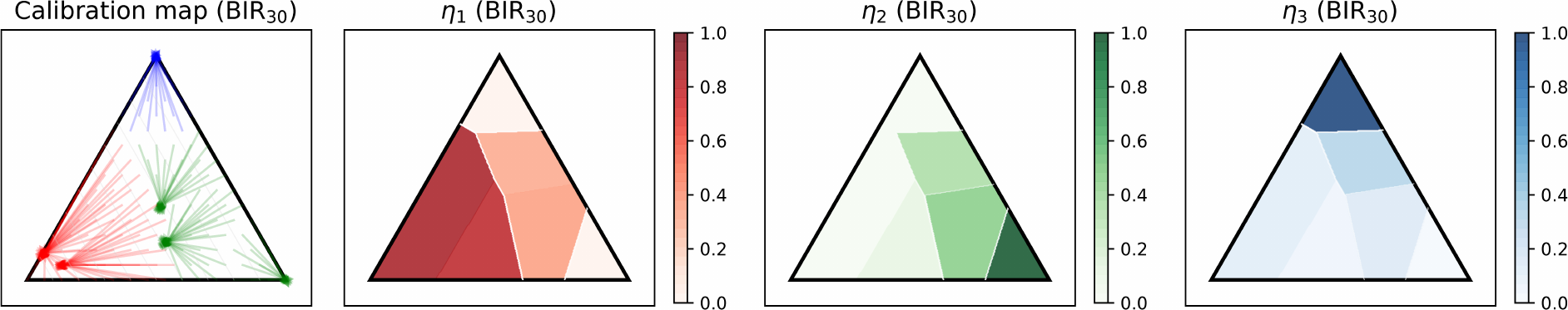} \\
  \includegraphics[width=\columnwidth]{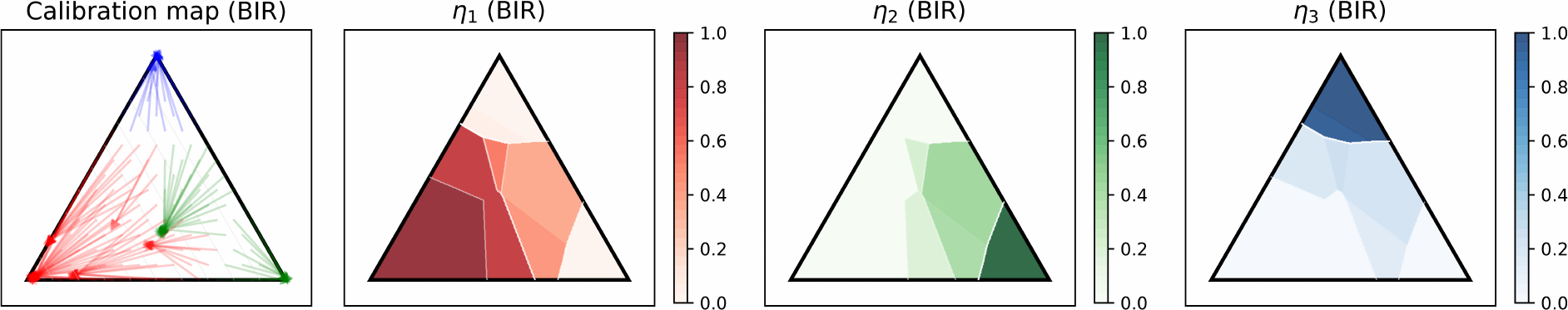} \\
  \includegraphics[width=\columnwidth]{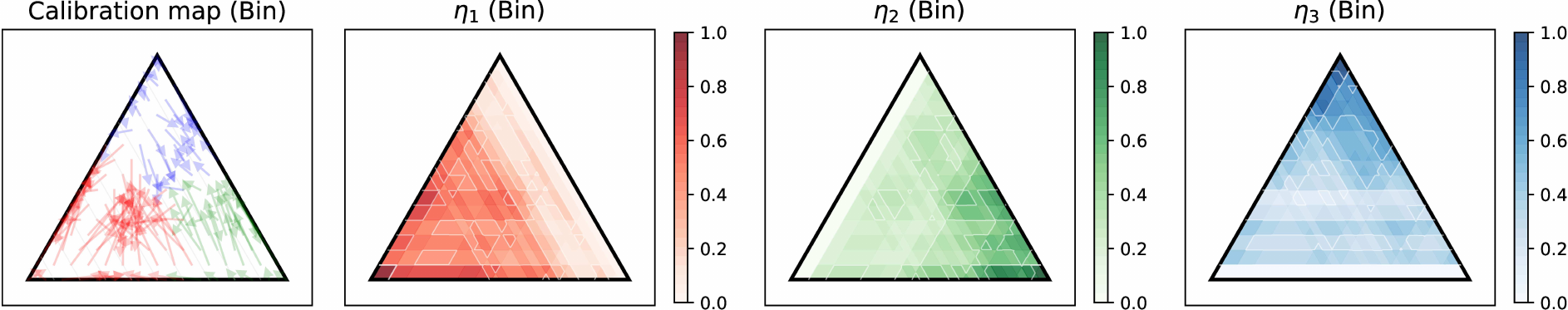} \\
  \includegraphics[width=\columnwidth]{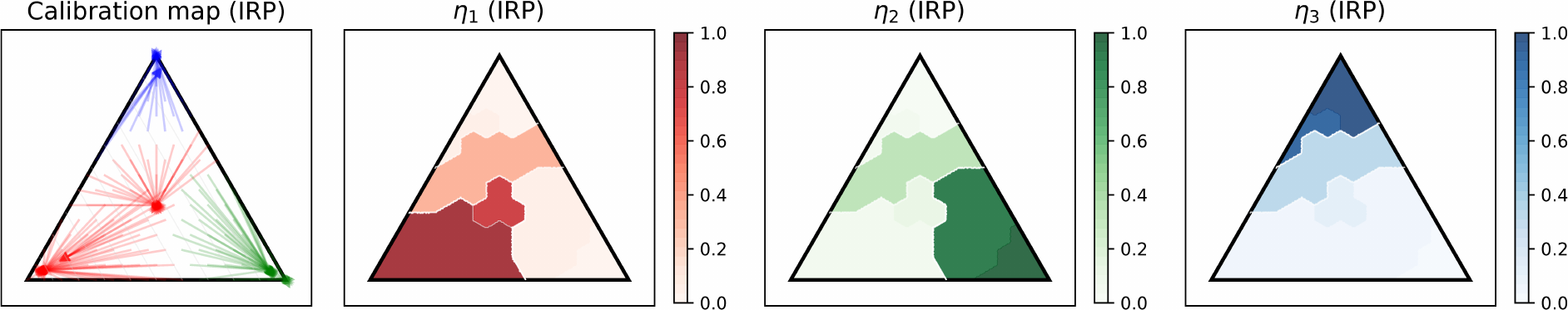} \\
  \includegraphics[width=\columnwidth]{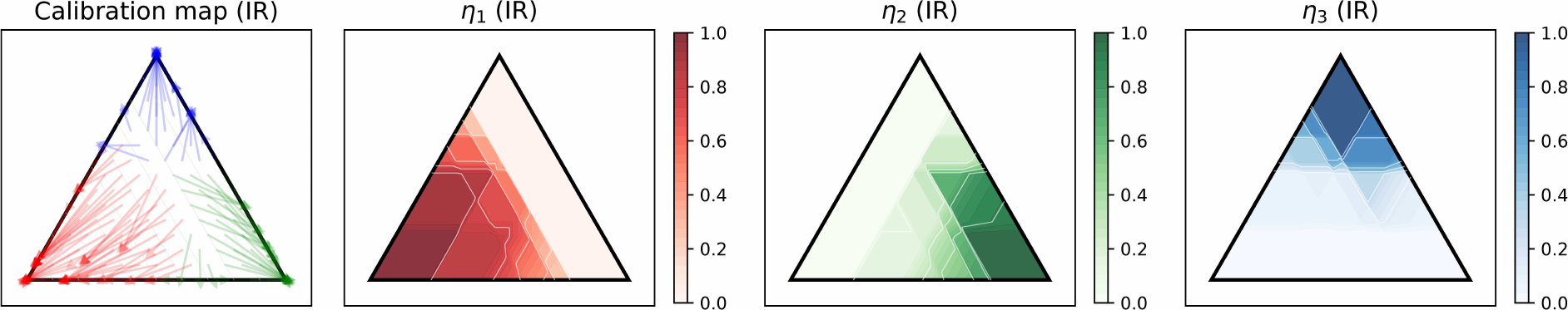}
  \caption{
    Calibration maps of nonparametric recalibrators.
    The base model is MLP and balance-scale dataset is used.
    In each row (corresponding to one recalibrator), we show the calibration map (vector field), the first/second/third coordinates of the calibration map, respectively, from left to right.
  }
  \label{figure:cm_nonparametric_mlp}
\end{figure}

\begin{figure}[t]
  \centering
  \includegraphics[width=\columnwidth]{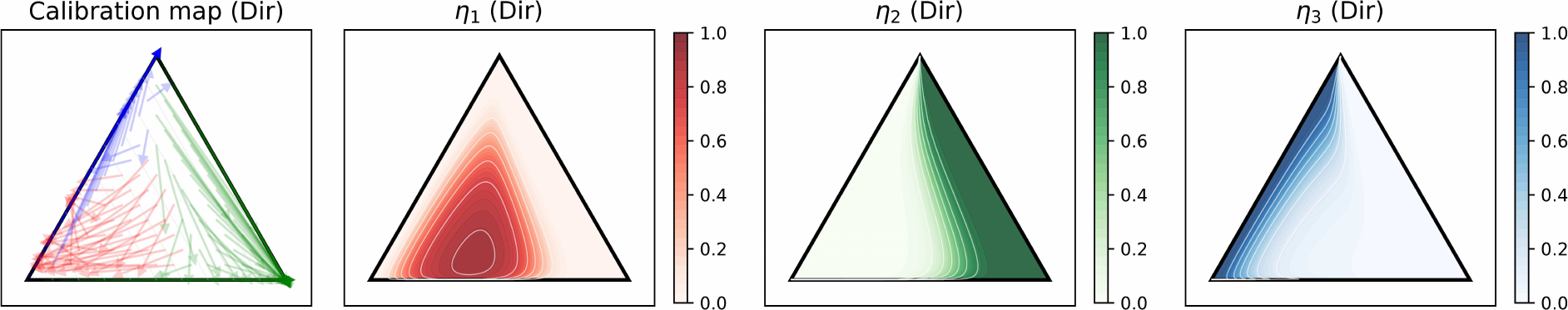} \\
  \includegraphics[width=\columnwidth]{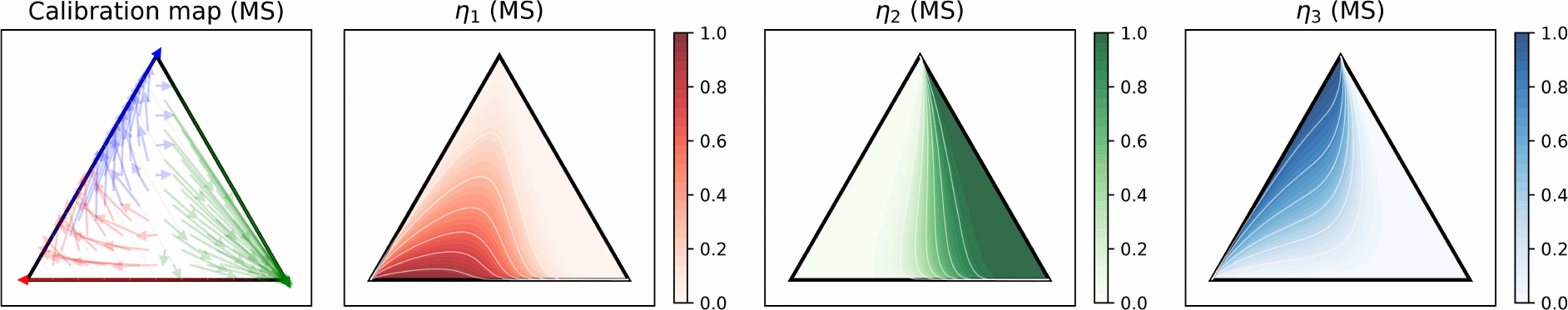} \\
  \includegraphics[width=\columnwidth]{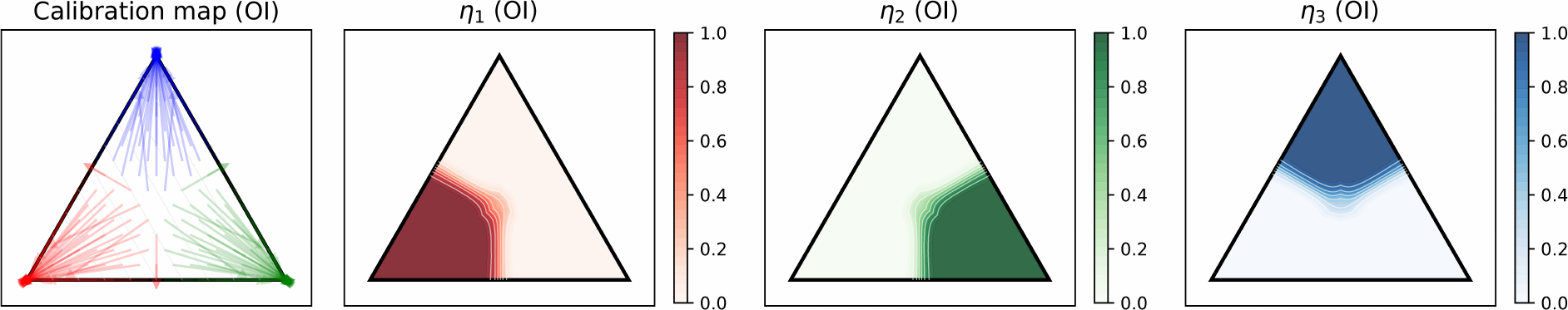} \\
  \includegraphics[width=\columnwidth]{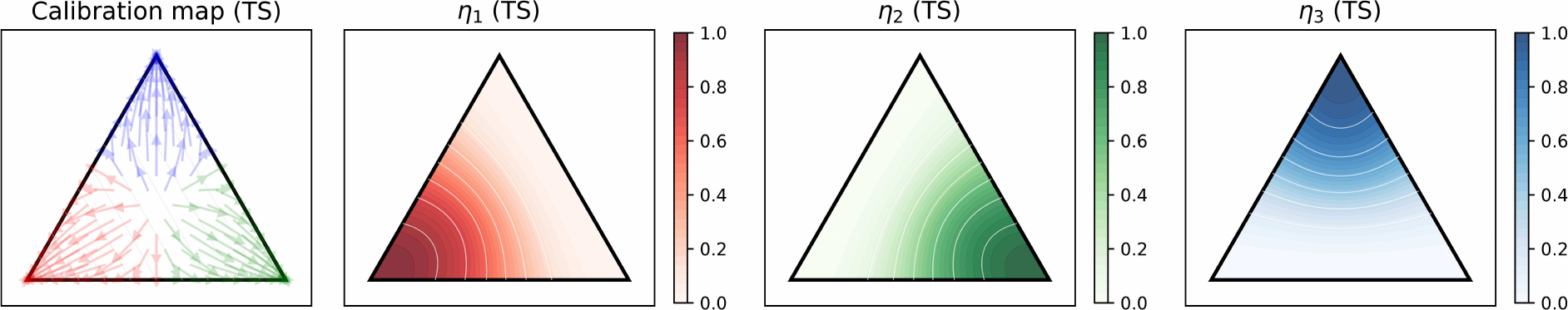}
  \caption{
    Calibration maps of parametric recalibrators.
    The base model is MLP and balance-scale dataset is used.
    In each row (corresponding to one recalibrator), we show the calibration map (vector field), the first/second/third coordinates of the calibration map, respectively, from left to right.
  }
  \label{figure:cm_parametric_mlp}
\end{figure}

\begin{figure}[t]
  \centering
  \includegraphics[width=\columnwidth]{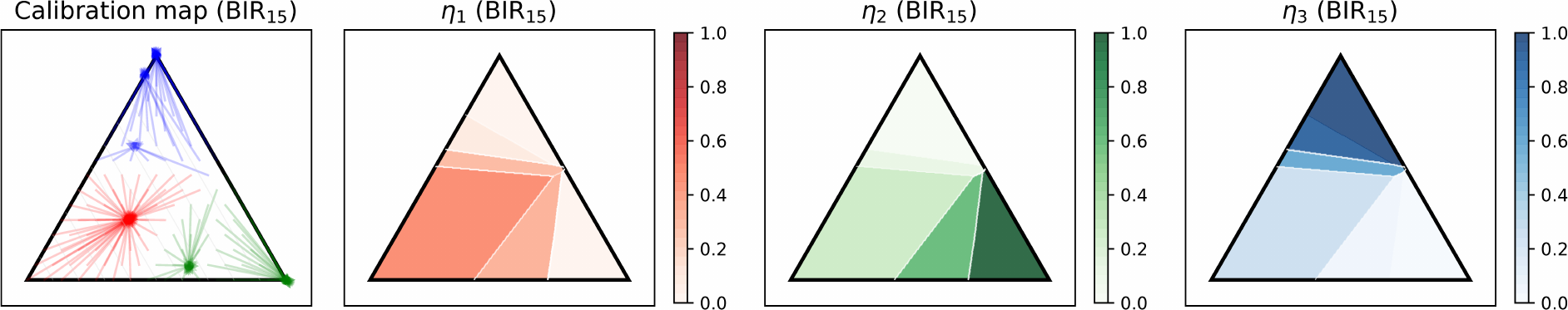} \\
  \includegraphics[width=\columnwidth]{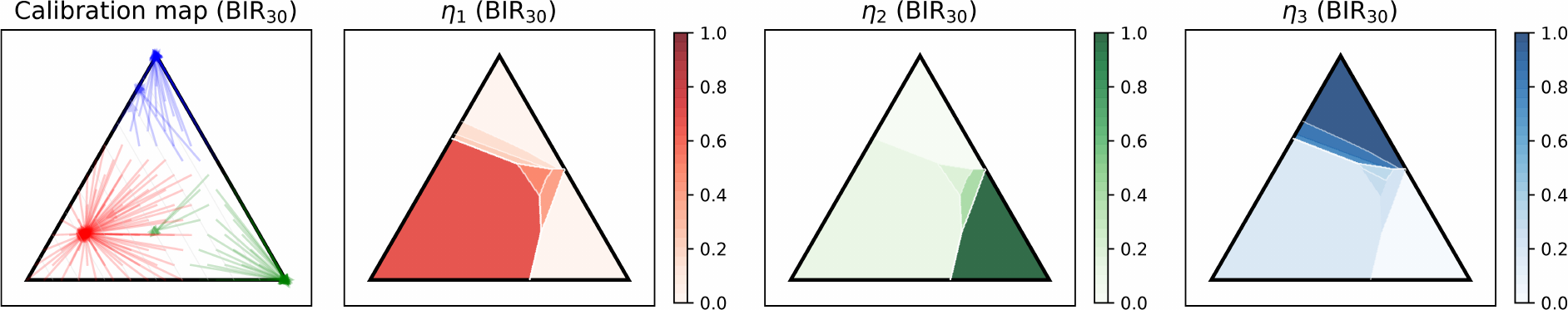} \\
  \includegraphics[width=\columnwidth]{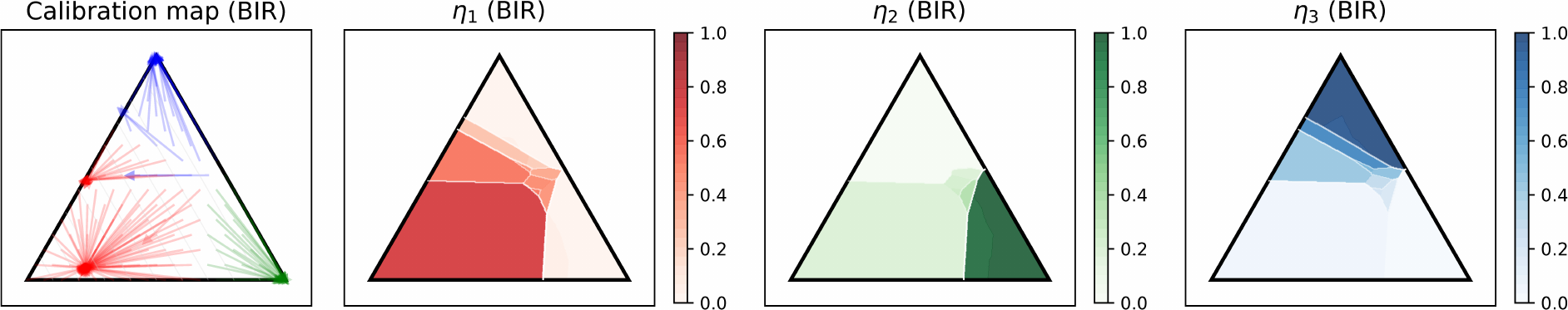} \\
  \includegraphics[width=\columnwidth]{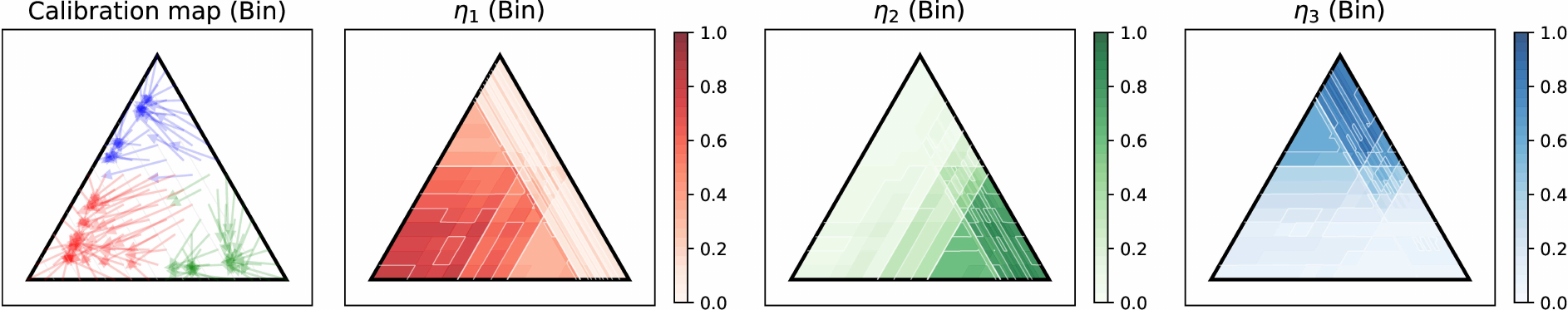} \\
  \includegraphics[width=\columnwidth]{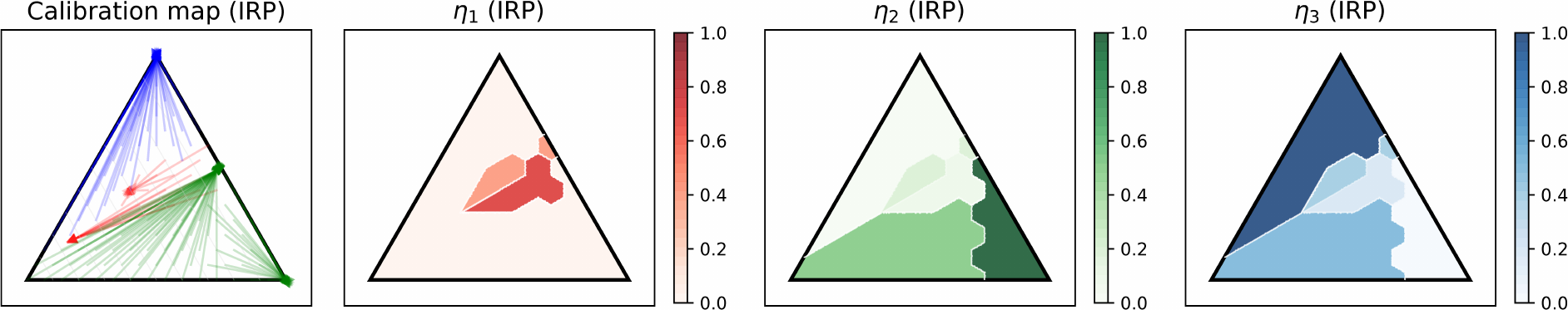} \\
  \includegraphics[width=\columnwidth]{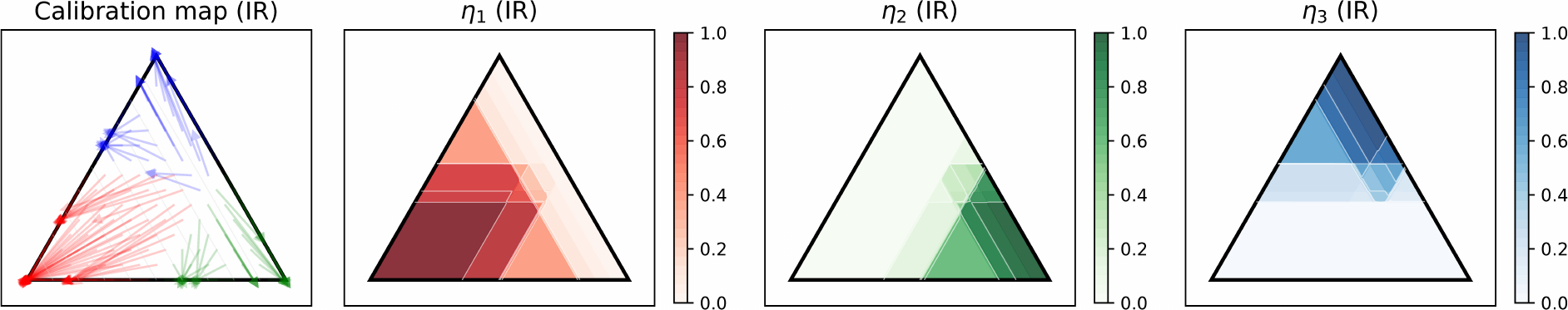}
  \caption{
    Calibration maps of nonparametric recalibrators.
    The base model is naive Bayes and balance-scale dataset is used.
    In each row (corresponding to one recalibrator), we show the calibration map (vector field), the first/second/third coordinates of the calibration map, respectively, from left to right.
  }
  \label{figure:cm_nonparametric_nbayes}
\end{figure}

\begin{figure}[t]
  \centering
  \includegraphics[width=\columnwidth]{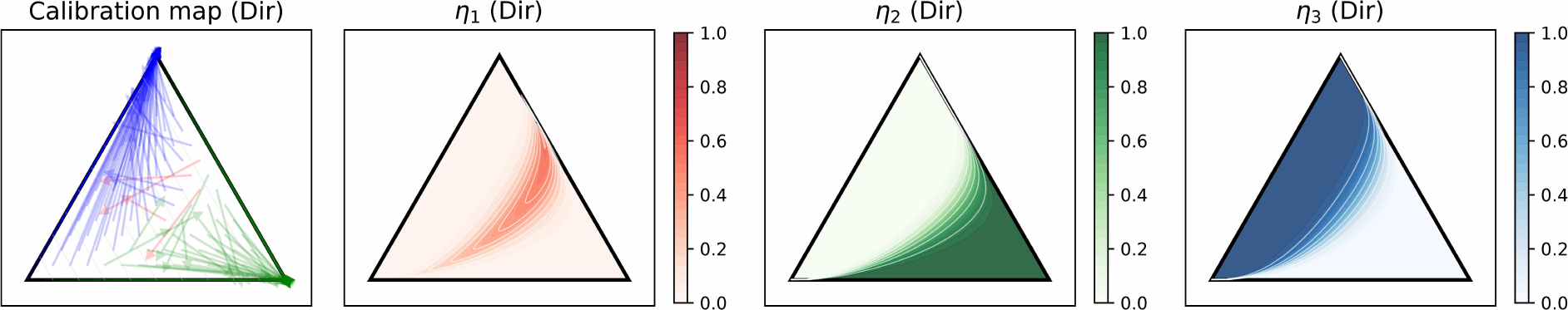} \\
  \includegraphics[width=\columnwidth]{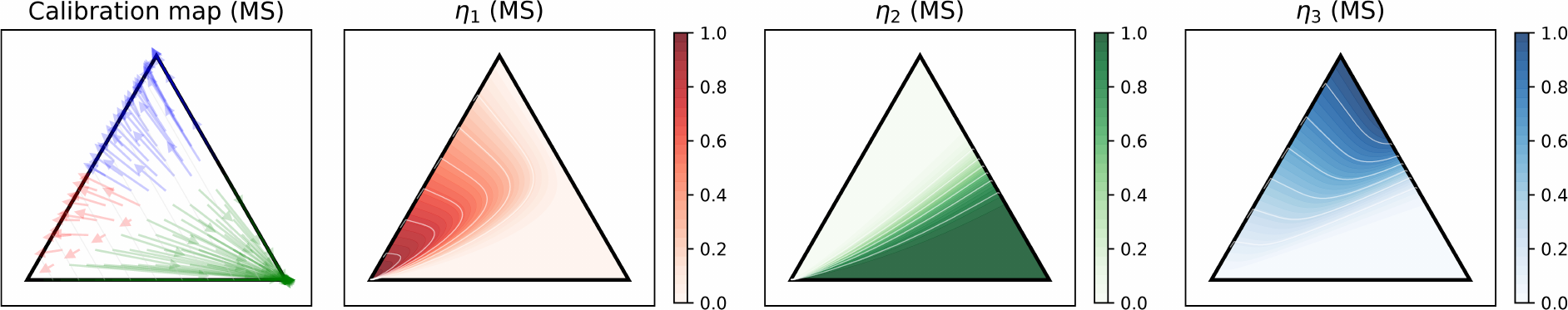} \\
  \includegraphics[width=\columnwidth]{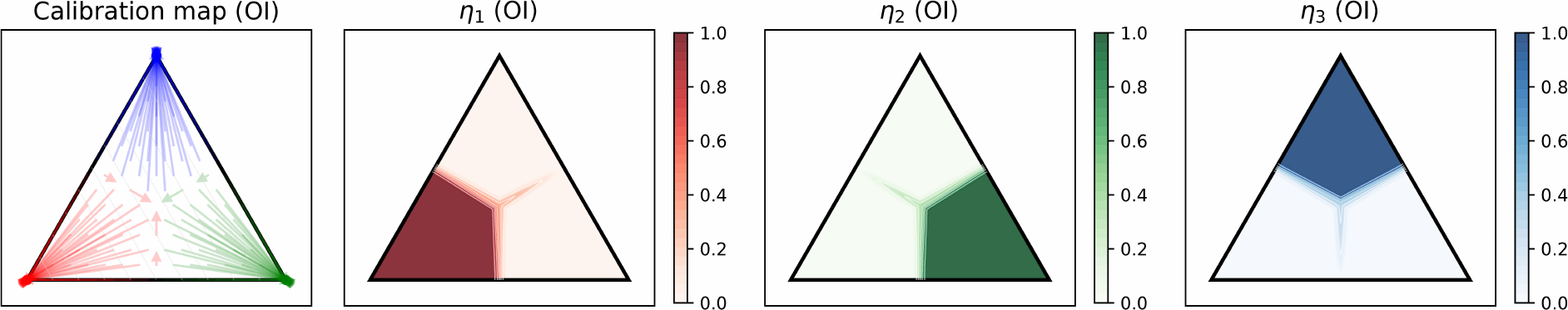} \\
  \includegraphics[width=\columnwidth]{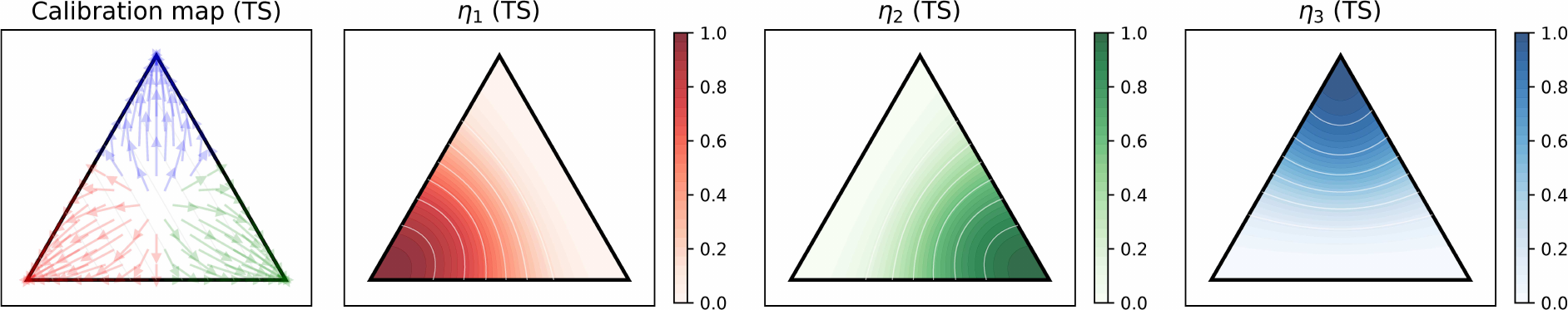}
  \caption{
    Calibration maps of parametric recalibrators.
    The base model is naive Bayes and balance-scale dataset is used.
    In each row (corresponding to one recalibrator), we show the calibration map (vector field), the first/second/third coordinates of the calibration map, respectively, from left to right.
  }
  \label{figure:cm_parametric_nbayes}
\end{figure}

\subsection{Computational time}
\label{section:experiment_detail:time}

We additionally measure the computational time to compare \method{BIR} and \method{IRP}, the two most promising approaches in probability calibration.
We report the computational time of fitting recalibrators and test prediction, averaged over $10$ trials.
As seen in \cref{table:time}, the running time of \method{BIR} becomes more advantageous with a mildly large number of classes (${\sim}10$).

\begin{table}[t]
  \centering
  \caption{Averaged running times (in seconds) over $10$ trials.}
  \label{table:time}
  \begin{tabular}{lrr rrrr r}
    \toprule
    \multirow{2}{*}{Dataset} & \multirow{2}{*}{\# of classes} & \multirow{2}{*}{Sample size} & \multicolumn{3}{c}{\method{BrenierIR} \tiny{(Ours)}} & \multirow{2}{*}{\method{IRP}} \\
    \cmidrule(lr){4-6}
    & & & $k=15$ & $k=30$ & $k=50$ & \\
    \midrule
    balance-scale & 3  & 625  & 3.60  & 13.69  & 41.19  & 0.29 \\
    car           & 4  & 1728 & 16.84 & 50.91  & 221.91 & 3.20 \\
    cleveland     & 5  & 297  & 4.68  & 16.52  & 35.05  & 4.26 \\
    dermatology   & 6  & 358  & 7.17  & 25.82  & 57.81  & 48.52 \\
    glass         & 6  & 214  & 4.51  & 18.26  & 55.76  & 7.00 \\
    segment       & 7  & 2310 & 39.02 & 119.03 & 279.27 & 124.82 \\
    vehicle       & 4  & 846  & 7.18  & 18.39  & 49.37  & 1.48 \\
    yeast         & 10 & 1484 & 27.72 & 104.01 & 256.01 & 3243.40 \\
    \bottomrule
  \end{tabular}
\end{table}

\subsection{More illustration of calibration maps}
\label{section:experiment_detail:calibration_map}
We show more illustrations of the calibration maps, complementing \cref{figure:cm} in \cref{section:calibration}.
In \cref{figure:cm_nonparametric_mlp}, we show the calibration maps of nonparametric recalibrators:
\method{BIR} (Brenier isotonic regression), \method{Bin} (OvR binning), \method{IRP} (iterative recursive splitting), and \method{IR} (OvR isotonic regression).
The OvR approaches \method{Bin} and \method{IR} do not learn inter-class information, and hence the resulting calibration maps exhibit contours parallel to the simplex boundary.
Moreover, these two methods do not significantly perform binning so that the counter maps have considerably many bins.
This is because the number of bins $k$ grows exponentially in the number of classes $d$, such as $k=O(k_0^d)$, suppose $k_0$ is the number of bins of a single base recalibration model of the whole OvR model.
This can cause overfitting, leading to inferior calibration performance, as seen in \cref{table:calib}.
To the contrary, \method{BIR} and \method{IRP} produce coarser and adaptive binning results.
One notable difference is that \method{IRP} is restricted to the simplex splitting perpendicular to the simplex boundary.
Indeed, each boundary of the calibration map levels is perpendicular to the simplex boundary.%
\footnote{
  Although the theoretical binning boundaries given by \method{IRP} are completely straight and perpendicular to the simplex boundaries,
  the computed boundaries in \cref{figure:cm_nonparametric_mlp} look a bit ``jaggy'' due to the insufficient resolution when we run \method{IRP} algorithm.
  For this visualization, we used $200$ grids to run \method{IRP} to get as close boundaries to theory as possible.
}
This is attributed to their simplex splitting criterion $R_k$ in \citet[Eq.~(3)]{Berta2024AISTATS}, and can also be seen in their Figure~3.
Our \method{BIR} does not have such a restriction because the simplex splitting is based on the Laguerre cells, which is fully adaptive to the Kantorovich potential.
As we increase the number of bins $k$ for \method{BIR}, the simplex binning results become finer, but not excessively detailed unlike \method{Bin} and \method{IR}.
Moreover, \method{BIR} has outputs clearly concentrated on a few number of points over the simplex, as can bee seen in the calibration map visualization,
which nicely trades off the bias and variance in practice.

In \cref{figure:cm_parametric_mlp}, we show the calibration maps of parametric recalibrators:
\method{Dir} (Dirichlet calibration), \method{MS} (matrix scaling), \method{OI} (order-invariant network), and \method{TS} (temperature scaling).
By comparing the calibration maps, the vector fields produced by these parametric recalibrators entail the non-degenerating behaviors, unlike the nonparametric recalibrators.
In the right three contour plots, the level-set boundaries are more smooth than those of the nonparametric recalibrators, which results in worse expressivity.

Further, we tested the same setup but with a different based model, naive Bayes, instead of MLP.
The obtained calibration maps for all recalibrators are shown in \cref{figure:cm_nonparametric_nbayes} (for nonparametric recalibrators) and \cref{figure:cm_parametric_nbayes} (for parametric recalibrators).
The overall trends resemble the MLP case.

\subsection{More calibration results}
\label{section:experiment_details:benchmark}
In addition to the calibration results shown in \cref{section:calibration},
which only measures the $L_1$ calibration error,
we expand in this section the benchmark results for the other metrics.
Specifically, \cref{table:calib:accuracy} shows accuracy, \cref{table:calib:classwise-ce} shows the classwise calibration error, and \cref{table:calib:conf-ce} shows the confidence calibration error.

We observe the following points.
\begin{enumerate}
  \item In terms of the classwise/confidence/$L_1$ calibration error, \method{IRP} is quite a powerful baseline, and our \method{BrenierIR} performs on par with it.

  \item In terms of accuracy, \method{IRP} is no longer strong enough but \method{BrenierIR} achieves the reasonable standard.
\end{enumerate}

\begin{table*}[t]
\centering
\caption{
  Recalibration results for MLP \textbf{(upper table)} and linear SVM \textbf{(lower table)}.
  Each number indicates \textbf{accuracy} (higher is better) with averaging $10$ trials, and bold-faced if the recalibrator achieves the \underline{best or second best} performance or statistically indistinguishable from them by the Mann--Whitney $U$ test (significance level: $5$\%).
}
\label{table:calib:accuracy}
\fontsize{8pt}{8pt}\selectfont
\begin{tabular}{lccccccccCCC}
\toprule
\multirow{2}{*}{\diagbox[width=10.5em]{Dataset}{Recalibrator}} & \multirow{2}{*}{---} & \multirow{2}{*}{\method{Bin}} & \multirow{2}{*}{\method{Dir}} & \multirow{2}{*}{\method{IRP}} & \multirow{2}{*}{\method{IR}} & \multirow{2}{*}{\method{MS}} & \multirow{2}{*}{\method{OI}} & \multirow{2}{*}{\method{TS}} & \multicolumn{3}{c}{\method{BrenierIR} \tiny{(Ours)}} \\
\cmidrule(lr){10-12}
&&&&&&&&& $k=15$ & $k=30$ & $k=50$ \\
\midrule
balance-scale & 0.928 & \textbf{0.968} & \textbf{0.968} & 0.960 & \textbf{0.968} & 0.928 & 0.928 & 0.928 & \textbf{0.965} & 0.960 & 0.959 \\
car & \textbf{0.991} & \textbf{0.988} & \textbf{0.988} & 0.986 & 0.986 & \textbf{0.991} & \textbf{0.991} & \textbf{0.991} & 0.965 & \textbf{0.988} & 0.988 \\
cleveland & 0.525 & 0.492 & \textbf{0.607} & 0.541 & 0.557 & \textbf{0.590} & 0.525 & 0.525 & 0.549 & 0.528 & 0.536 \\
dermatology & 0.919 & 0.919 & \textbf{0.932} & 0.919 & \textbf{0.932} & 0.919 & 0.919 & 0.919 & 0.916 & \textbf{0.934} & 0.923 \\
glass & 0.698 & 0.674 & 0.698 & 0.698 & \textbf{0.791} & 0.698 & 0.698 & 0.698 & \textbf{0.744} & 0.730 & 0.702 \\
vehicle & 0.829 & 0.824 & 0.829 & 0.812 & \textbf{0.835} & 0.824 & 0.829 & 0.829 & \textbf{0.844} & 0.809 & 0.829 \\
\bottomrule
\end{tabular} \\
\vspace{3pt}
\begin{tabular}{lccccccccCCC}
\toprule
\multirow{2}{*}{\diagbox[width=10.5em]{Dataset}{Recalibrator}} & \multirow{2}{*}{---} & \multirow{2}{*}{\method{Bin}} & \multirow{2}{*}{\method{Dir}} & \multirow{2}{*}{\method{IRP}} & \multirow{2}{*}{\method{IR}} & \multirow{2}{*}{\method{MS}} & \multirow{2}{*}{\method{OI}} & \multirow{2}{*}{\method{TS}} & \multicolumn{3}{c}{\method{BrenierIR} \tiny{(Ours)}} \\
\cmidrule(lr){10-12}
&&&&&&&&& $k=15$ & $k=30$ & $k=50$ \\
\midrule
balance-scale & 0.872 & \textbf{0.880} & 0.736 & 0.456 & 0.872 & 0.736 & 0.736 & 0.736 & \textbf{0.928} & \textbf{0.928} & \textbf{0.928} \\
car & 0.882 & \textbf{0.893} & 0.737 & 0.882 & \textbf{0.908} & 0.754 & 0.260 & 0.260 & 0.860 & 0.881 & 0.879 \\
cleveland & \textbf{0.607} & 0.541 & 0.574 & 0.590 & 0.557 & 0.590 & 0.557 & 0.557 & 0.595 & \textbf{0.597} & 0.582 \\
dermatology & \textbf{0.959} & \textbf{0.959} & 0.568 & 0.838 & \textbf{0.959} & 0.568 & 0.311 & 0.311 & 0.915 & 0.934 & \textbf{0.950} \\
glass & 0.605 & 0.581 & 0.419 & 0.349 & \textbf{0.651} & 0.395 & 0.326 & 0.326 & \textbf{0.577} & \textbf{0.658} & \textbf{0.644} \\
vehicle & \textbf{0.765} & \textbf{0.759} & 0.576 & 0.259 & \textbf{0.765} & 0.588 & 0.412 & 0.412 & 0.749 & 0.733 & 0.747 \\
\bottomrule
\end{tabular}
\end{table*}

\begin{table*}[t]
\centering
\caption{
  Recalibration results for MLP \textbf{(upper table)} and linear SVM \textbf{(lower table)}.
  Each number indicates the \textbf{classwise calibration error} (lower is better) with averaging $10$ trials, and bold-faced if the recalibrator achieves the \underline{best or second best} performance or statistically indistinguishable from them by the Mann--Whitney $U$ test (significance level: $5$\%).
}
\label{table:calib:classwise-ce}
\fontsize{8pt}{8pt}\selectfont
\begin{tabular}{lccccccccCCC}
\toprule
\multirow{2}{*}{\diagbox[width=10.5em]{Dataset}{Recalibrator}} & \multirow{2}{*}{---} & \multirow{2}{*}{\method{Bin}} & \multirow{2}{*}{\method{Dir}} & \multirow{2}{*}{\method{IRP}} & \multirow{2}{*}{\method{IR}} & \multirow{2}{*}{\method{MS}} & \multirow{2}{*}{\method{OI}} & \multirow{2}{*}{\method{TS}} & \multicolumn{3}{c}{\method{BrenierIR} \tiny{(Ours)}} \\
\cmidrule(lr){10-12}
&&&&&&&&& $k=15$ & $k=30$ & $k=50$ \\
\midrule
balance-scale & 0.059 & 0.050 & 0.029 & \textbf{0.020} & 0.038 & 0.045 & 0.046 & 0.039 & \textbf{0.020} & \textbf{0.020} & 0.024 \\
car & 0.015 & 0.011 & \textbf{0.007} & \textbf{0.007} & 0.008 & 0.008 & 0.028 & 0.008 & 0.011 & 0.010 & 0.009 \\
cleveland & 0.114 & 0.089 & 0.081 & \textbf{0.035} & \textbf{0.063} & 0.094 & 0.184 & 0.097 & 0.072 & 0.083 & 0.087 \\
dermatology & 0.028 & 0.030 & 0.024 & 0.029 & \textbf{0.023} & 0.027 & 0.069 & 0.027 & \textbf{0.019} & 0.026 & 0.027 \\
glass & 0.097 & 0.087 & 0.089 & \textbf{0.072} & 0.082 & 0.087 & 0.132 & 0.098 & 0.081 & 0.085 & \textbf{0.071} \\
vehicle & 0.058 & 0.052 & 0.040 & \textbf{0.021} & 0.040 & 0.057 & 0.116 & 0.055 & 0.044 & \textbf{0.027} & 0.044 \\
\bottomrule
\end{tabular} \\
\vspace{3pt}
\begin{tabular}{lccccccccCCC}
\toprule
\multirow{2}{*}{\diagbox[width=10.5em]{Dataset}{Recalibrator}} & \multirow{2}{*}{---} & \multirow{2}{*}{\method{Bin}} & \multirow{2}{*}{\method{Dir}} & \multirow{2}{*}{\method{IRP}} & \multirow{2}{*}{\method{IR}} & \multirow{2}{*}{\method{MS}} & \multirow{2}{*}{\method{OI}} & \multirow{2}{*}{\method{TS}} & \multicolumn{3}{c}{\method{BrenierIR} \tiny{(Ours)}} \\
\cmidrule(lr){10-12}
&&&&&&&&& $k=15$ & $k=30$ & $k=50$ \\
\midrule
balance-scale & 0.032 & 0.061 & 0.094 & \textbf{0.004} & 0.089 & 0.053 & 0.179 & 0.053 & \textbf{0.029} & 0.029 & 0.034 \\
car & \textbf{0.024} & 0.065 & 0.054 & 0.139 & 0.052 & 0.052 & 0.177 & 0.133 & \textbf{0.028} & 0.041 & 0.035 \\
cleveland & \textbf{0.074} & 0.076 & 0.083 & \textbf{0.037} & 0.089 & 0.106 & 0.139 & 0.154 & 0.082 & 0.097 & 0.093 \\
dermatology & 0.040 & 0.041 & 0.029 & 0.051 & \textbf{0.013} & 0.024 & 0.063 & 0.098 & 0.021 & 0.020 & \textbf{0.020} \\
glass & 0.089 & 0.073 & 0.105 & \textbf{0.007} & \textbf{0.067} & 0.086 & 0.096 & 0.083 & 0.074 & 0.089 & 0.100 \\
vehicle & 0.074 & 0.082 & 0.104 & \textbf{0.002} & 0.066 & 0.084 & 0.098 & 0.087 & \textbf{0.061} & 0.072 & 0.070 \\
\bottomrule
\end{tabular}
\end{table*}

\begin{table*}[t]
\centering
\caption{
  Recalibration results for MLP \textbf{(upper table)} and linear SVM \textbf{(lower table)}.
  Each number indicates the \textbf{confidence calibration error} (lower is better) with averaging $10$ trials, and bold-faced if the recalibrator achieves the \underline{best or second best} performance or statistically indistinguishable from them by the Mann--Whitney $U$ test (significance level: $5$\%).
}
\label{table:calib:conf-ce}
\fontsize{8pt}{8pt}\selectfont
\begin{tabular}{lccccccccCCC}
\toprule
\multirow{2}{*}{\diagbox[width=10.5em]{Dataset}{Recalibrator}} & \multirow{2}{*}{---} & \multirow{2}{*}{\method{Bin}} & \multirow{2}{*}{\method{Dir}} & \multirow{2}{*}{\method{IRP}} & \multirow{2}{*}{\method{IR}} & \multirow{2}{*}{\method{MS}} & \multirow{2}{*}{\method{OI}} & \multirow{2}{*}{\method{TS}} & \multicolumn{3}{c}{\method{BrenierIR} \tiny{(Ours)}} \\
\cmidrule(lr){10-12}
&&&&&&&&& $k=15$ & $k=30$ & $k=50$ \\
\midrule
balance-scale & 0.062 & 0.062 & 0.023 & \textbf{0.010} & 0.032 & 0.035 & 0.049 & 0.044 & \textbf{0.012} & 0.020 & 0.023 \\
car & 0.024 & 0.016 & 0.011 & 0.013 & \textbf{0.011} & \textbf{0.011} & 0.063 & 0.015 & 0.019 & 0.017 & 0.016 \\
cleveland & 0.172 & 0.201 & 0.106 & \textbf{0.061} & 0.158 & \textbf{0.099} & 0.435 & 0.126 & 0.119 & 0.168 & 0.161 \\
dermatology & 0.058 & \textbf{0.050} & 0.066 & 0.083 & 0.058 & 0.067 & 0.418 & 0.058 & \textbf{0.024} & 0.058 & 0.057 \\
glass & 0.155 & 0.155 & 0.218 & 0.197 & 0.228 & 0.169 & 0.517 & \textbf{0.136} & 0.152 & \textbf{0.159} & \textbf{0.144} \\
vehicle & 0.064 & 0.064 & 0.063 & \textbf{0.040} & 0.057 & \textbf{0.050} & 0.214 & 0.051 & 0.078 & 0.055 & 0.062 \\
\bottomrule
\end{tabular} \\
\vspace{3pt}
\begin{tabular}{lccccccccCCC}
\toprule
\multirow{2}{*}{\diagbox[width=10.5em]{Dataset}{Recalibrator}} & \multirow{2}{*}{---} & \multirow{2}{*}{\method{Bin}} & \multirow{2}{*}{\method{Dir}} & \multirow{2}{*}{\method{IRP}} & \multirow{2}{*}{\method{IR}} & \multirow{2}{*}{\method{MS}} & \multirow{2}{*}{\method{OI}} & \multirow{2}{*}{\method{TS}} & \multicolumn{3}{c}{\method{BrenierIR} \tiny{(Ours)}} \\
\cmidrule(lr){10-12}
&&&&&&&&& $k=15$ & $k=30$ & $k=50$ \\
\midrule
balance-scale & \textbf{0.030} & 0.075 & 0.141 & \textbf{0.006} & 0.098 & 0.051 & 0.221 & 0.048 & 0.042 & 0.038 & 0.038 \\
car & \textbf{0.026} & 0.101 & 0.088 & 0.277 & 0.086 & 0.059 & 0.092 & 0.218 & \textbf{0.046} & 0.062 & 0.049 \\
cleveland & 0.148 & 0.183 & 0.209 & \textbf{0.030} & 0.158 & 0.192 & 0.300 & 0.276 & 0.153 & \textbf{0.147} & 0.160 \\
dermatology & 0.087 & 0.083 & 0.224 & 0.100 & \textbf{0.033} & 0.064 & 0.102 & 0.061 & 0.054 & 0.043 & \textbf{0.041} \\
glass & 0.182 & 0.156 & 0.102 & \textbf{0.008} & 0.150 & 0.150 & 0.120 & \textbf{0.076} & 0.217 & 0.225 & 0.220 \\
vehicle & 0.106 & 0.099 & 0.138 & \textbf{0.001} & 0.088 & 0.084 & \textbf{0.058} & 0.101 & 0.095 & 0.100 & 0.115 \\
\bottomrule
\end{tabular}
\end{table*}

\end{document}